\documentclass{article}
\usepackage{iclr2019_conference,times}

\usepackage{amsmath,amsfonts,bm}

\def\eqref#1{equation~\ref{#1}}
\def\1{\bm{1}}

\DeclareMathAlphabet{\mathsfit}{\encodingdefault}{\sfdefault}{m}{sl}
\SetMathAlphabet{\mathsfit}{bold}{\encodingdefault}{\sfdefault}{bx}{n}

\usepackage{hyperref}
\usepackage{url}

\usepackage{floatrow}
\newfloatcommand{capbtabbox}{table}[][\FBwidth]

\title{Hyperbolic Discounting and Learning over Multiple Horizons}

\author{William Fedus\\
Google Brain \\
University of Montreal (Mila)\\
\texttt{liamfedus@google.com}
\And
Carles Gelada \\
Google Brain\\
\texttt{cgel@google.com}\\
\And 
Yoshua Bengio \\
University of Montreal (Mila) \\ 
CIFAR Senior Fellow\\
\texttt{yoshua.bengio@mila.quebec}\\
\AND
Marc G. Bellemare\\
Google Brain\\
\texttt{bellemare@google.com }\\
\And
Hugo Larochelle \\
Google Brain \\
\texttt{hugolarochelle@google.com}\\
}

\usepackage{graphicx}
\usepackage{subfigure}
\usepackage[utf8]{inputenc} % allow utf-8 input
\usepackage[T1]{fontenc}    % use 8-bit T1 fonts
\usepackage{hyperref}       % hyperlinks
\usepackage{url}            % simple URL typesetting
\usepackage{booktabs}       % professional-quality tables
\usepackage{amsfonts}       % blackboard math symbols
\usepackage{nicefrac}       % compact symbols for 1/2, etc.
\usepackage{microtype}      % microtypography
\usepackage{wrapfig}

\usepackage{amsthm}
\usepackage{algpseudocode}% http://ctan.org/pkg/algorithmicx
 
\theoremstyle{definition}
\newtheorem{theorem}{Theorem}[section]
\newtheorem{lemma}[theorem]{Lemma}
\newtheorem{definition}{Definition}[section]
\usepackage{multirow}
\usepackage{caption}

\iclrfinalcopy % Uncomment for camera-ready version, but NOT for submission.
\begin{document}
\maketitle

\begin{abstract}
Reinforcement learning (RL) typically defines a discount factor ($\gamma$) as part of the Markov Decision Process.
The discount factor values future rewards by an exponential scheme that leads to theoretical convergence guarantees of the Bellman equation.
However, evidence from psychology, economics and neuroscience suggests that humans and animals instead have \emph{hyperbolic} time-preferences ($\frac{1}{1 + kt}$ for $k>0$).  
In this work we revisit the fundamentals of discounting in RL and bridge this disconnect by implementing an RL agent that acts via hyperbolic discounting.
We demonstrate that a simple approach approximates hyperbolic discount functions while still using familiar temporal-difference learning techniques in RL.  
Additionally, and independent of hyperbolic discounting, we make a surprising discovery that simultaneously learning value functions over multiple time-horizons is an effective auxiliary task which often improves over a strong value-based RL agent, Rainbow.
\end{abstract}

\section{Introduction}

The standard treatment of the reinforcement learning (RL) problem is the Markov Decision Process (MDP) which includes a discount factor $0 \leq \gamma  < 1$ that exponentially reduces the present value of future rewards \citep{bellman1957markovian, sutton2018reinforcement}. 
A reward $r_t$ received in $t$-time steps is devalued to $\gamma^{t}r_t$, a discounted utility model introduced by \citet*{samuelson1937note}.
This establishes a time-preference for rewards realized sooner rather than later.  
The decision to exponentially discount future rewards by $\gamma$ leads to value functions that satisfy theoretical convergence properties \citep{bertsekas1995neuro}.
The magnitude of $\gamma$ also plays a role in stabilizing learning dynamics of RL algorithms \citep{prokhorov1997adaptive, bertsekas1996neuro} and has recently been treated as a hyperparameter of the optimization \citep{OpenAI_dota, xu2018meta}.

However, both the magnitude and the functional form of this discounting function implicitly establish priors over the solutions learned.
The magnitude of $\gamma$ chosen establishes an \emph{effective horizon} for the agent, far beyond which rewards are neglected \citep{kearns2002near}.  
This effectively imposes a time-scale of the environment, which may not be accurate.
However, less well-known and expanded on later, the exponential discounting of future rewards is consistent with a prior belief that there exists a \emph{known} constant risk to the agent in the environment (\citet*{sozou1998hyperbolic}, Section \ref{sec: single_hazard}). 
This is a strong assumption that may not be supported in richer environments.

Additionally, discounting future values exponentially and according to a single discount factor $\gamma$ does not harmonize with the measured value preferences in humans and animals \citep{mazur1985probability, mazur1997choice, ainslie1992picoeconomics, green2004discounting, maia2009reinforcement}. 
A wealth of empirical evidence has been amassed that humans, monkeys, rats and pigeons instead discount future returns \emph{hyperbolically}, where $d_k(t) = \frac{1}{1+kt}$, for some positive $k > 0$ \citep{ainslie1975specious, ainslie1992picoeconomics, mazur1985probability, mazur1997choice, frederick2002time, green1981preference, green2004discounting}.

\begin{figure}[h!]
    \centering
    \includegraphics[width=0.7\columnwidth]{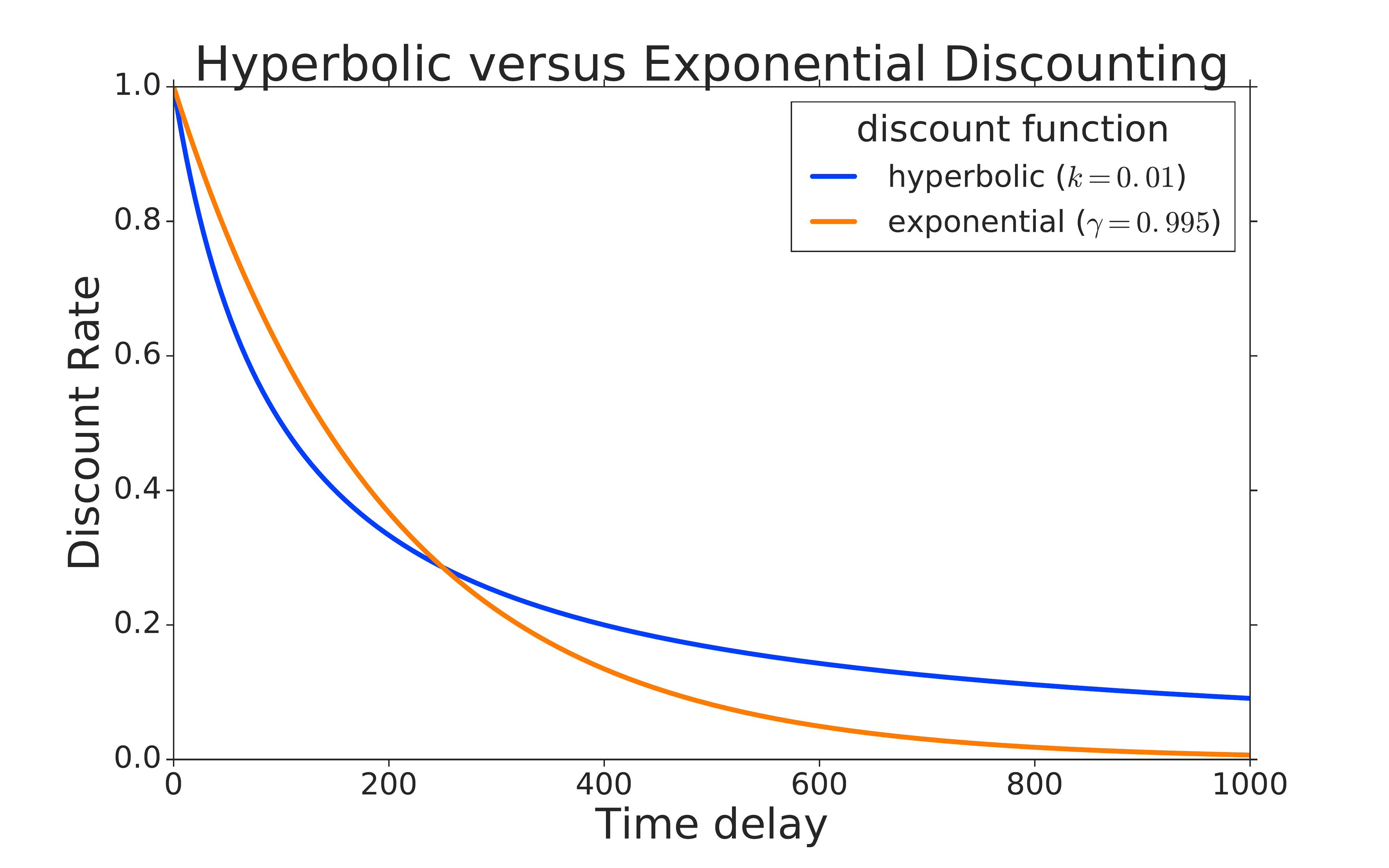}
    \caption{Hyperbolic versus exponential discounting.  Humans and animals often exhibit hyperbolic discounts (blue curve) which have shallower discount declines for large horizons. In contrast, RL agents often optimize exponential discounts (orange curve) which drop at a constant rate regardless of how distant the return.}
\end{figure}

As an example of hyperbolic time-preferences, consider the hypothetical:  a stranger approaches with a simple proposition.  
He offers you \$1M immediately with no risk, but if you can wait until tomorrow, he promises you \$1.1M dollars.  
With no further information many are skeptical of this would-be benefactor and choose to receive \$1M immediately.
Most rightly believe the future promise holds risk.
% However, preferences often reverse if the proposition is modified slightly.  
However, in an alternative proposition, he instead promises you \$1M in 365 days or \$1.1M in 366 days. 
Under these new terms many will instead choose the \$1.1M offer.
Effectively, the discount rate has \emph{decreased} further out, indicating the belief that it is less likely for the promise to be reneged on the 366th day if it were not already broken on the 365th day.
Note that discount rates in humans have been demonstrated to vary with the size of the reward so this time-reversal might not emerge for \$1 versus \$1.1 \citep{myerson1995discounting, green1997rate}.

Hyperbolic discounting is consistent with these reversals in time-preferences \citep{green1994temporal}.
Exponential discounting, on the other hand, always remains consistent between these choices and was shown in \citet*{strotz1955myopia} to be the only time-consistent sliding discount function. 
% Under these propositions, an exponentially discounting agent that prefers \$1M today versus \$1.1M tomorrow will similarly prefer \$1M in 365 days to \$1.1M in 366 days.
This discrepancy between the time-preferences of animals from the exponential discounted measure of value might be presumed irrational. 
However, \citet*{sozou1998hyperbolic} demonstrates that this behavior is mathematically consistent with the agent maintaining some uncertainty over the \emph{hazard rate} in the environment.
In this formulation, rewards are discounted based on the possibility the agent will succumb to a risk and will thus not survive to collect them.
Hazard rate, defined in Section \ref{sec: hazard_implies_discount}, measures the per-time-step risk the agent incurs as it acts in the environment.

\textbf{Hazard and its associated discount function.}
Common RL environments are also characterized by risk, but in a narrower sense.
In deterministic environments like the original Arcade Learning Environment (ALE) \citep{bellemare2013arcade} stochasticity is often introduced through techniques like no-ops \citep{mnih2015human} and sticky actions \citep{machado2018revisiting} where the action execution is noisy.
Physics simulators may have noise and the randomness of the policy itself induces risk.
But even with these stochastic injections the risk to reward emerges in a more restricted sense.
Episode-to-episode risk may vary as the value function and resulting policy evolve.
States once safely navigable may become dangerous through catastrophic forgetting \citep{mccloskey1989catastrophic, french1999catastrophic} or through exploration the agent may venture to new dangerous areas of the state space.
However, this is still a narrow manifestation of risk as the environment is generally stable and repetitive.
% Although risks can take many forms: a promise being broken, the episode terminating or the environment being different than expected, they all can be reduced to a probability of the reward not being realized in the environment. 
% Mathematically, it will be convenient to study hazard rates $\lambda$, which define a constant probability that the agent would die each time step of $e^{-\lambda}$ which we describe further in Section \ref{sec: hazard_implies_discount}.
In Section \ref{sec: hazard_mdps} we show that a prior distribution reflecting the uncertainty over the hazard rate, has an associated discount function in the sense that an MDP with either this hazard distribution or the discount function, has the same value function for all policies. 
This equivalence implies that learning policies with a discount function can be interpreted as making them robust to the associated hazard distribution. 
Thus, discounting serves as a tool to ensure that policies deployed in the real world perform well even under risks they were not trained under. 

\textbf{Hyperbolic discounting from TD-learning algorithms.}
We propose an algorithm that approximates hyperbolic discounting while building on successful Q-learning \citep{watkins1992q} tools and their associated theoretical guarantees.
We show learning many Q-values, each discounting exponentially with a different discount factor $\gamma$, can be aggregated to approximate hyperbolic (and other non-exponential) discount factors.
We demonstrate the efficacy of our approximation scheme in our proposed Pathworld environment which is characterized both by an uncertain per-time-step risk to the agent.  
The agent must choose which risky path to follow but it stands to gain a higher reward the longer, riskier paths.
A conceptually similar situation might arise for a foraging agent balancing easily realizable, small meals versus more distant, fruitful meals.  The setup is described in further detail in Section \ref{sec: hyperbolic_results}.
We then consider higher-dimensional RL agents in the ALE, where we measure the benefits of our technique.
Our approximation mirrors the work of \cite{kurth2009temporal, redish2010neural} which empirically demonstrates that modeling a finite set of $\mu$Agents simultaneously can approximate hyperbolic discounting function which is consistent with fMRI studies \citep{tanaka2004prediction, schweighofer2008low}.
Our method extends to other non-hyperbolic discount functions and uses deep neural networks to model the different Q-values from a shared representation.

% \begin{figure}[!ht]
%     \centering
%     \includegraphics[width=\columnwidth]{example-image-a.pdf}
%     \caption{Simple figure about modeling multiple horizons.}
%     \label{fig: multiple_horizons}
% \end{figure}

Surprisingly and in addition to enabling new discounting schemes, we observe that learning a set of Q-values is beneficial as an auxiliary task \citep{jaderberg2016reinforcement}.  
Adding this \emph{multi-horizon auxiliary task} often improves over strong baselines including C51 \citep{bellemare2017distributional} and Rainbow \citep{hessel2018rainbow} in the ALE \citep{bellemare2013arcade}.

% Through this work we challenge the accepted paradigm of modeling the environment with a single gamma. We present multi-horizon modeling and HyperbolicRainbow.  

The paper is organized as follows.
Section \ref{sec: hazard_implies_discount} recounts how a prior belief of the risk in the environment can imply a specific discount function.
Section \ref{sec: hazard_mdps} formalizes hazard in MDPs.
In Section \ref{sec: est_discount_through_gamma} we demonstrate that hyperbolic (and other) discounting rates can be computed by Q-learning \citep{watkins1992q} over \emph{multiple horizons}, that is, multiple discount functions $\gamma$.
We then provide a \emph{practical} approach to approximating these alternative discount schemes in Section \ref{sec: approx_hyperbolic_qvals}.
We demonstrate the efficacy of our approximation scheme in the Pathworld environment in Section \ref{sec: hyperbolic_results} and then go on to consider the high-dimensional ALE setting in Sections \ref{sec: hyperbolic_results}, \ref{sec: multi_horizon}.
We conclude with ablation studies, discussion and commentary on future research directions.

This work questions the RL paradigm of learning policies through a single discount function which exponentially discounts future rewards through two contributions:

\begin{enumerate}
    \item \textbf{Hyperbolic (and other non-exponential)-agent.}  A practical approach for training an agent which discounts future rewards by a hyperbolic (or other non-exponential) discount function and acts according to this.
    \item \textbf{Multi-horizon auxiliary task.}  A demonstration of multi-horizon learning over many $\gamma$ simultaneously as an effective auxiliary task.
\end{enumerate}

%%%%%%%%%%%%%%%%%%%%%%%%%%%%%%%%%%%%%%%%%%%%%%%%%%%%%%%%%%%%%
%                      Related work
%%%%%%%%%%%%%%%%%%%%%%%%%%%%%%%%%%%%%%%%%%%%%%%%%%%%%%%%%%%%%
\section{Related Work}
\textbf{Hyperbolic discounting in economics.}
Hyperbolic discounting is well-studied in the field of economics \citep{sozou1998hyperbolic, dasgupta2005uncertainty}.
\citet*{dasgupta2005uncertainty} proposes a softer interpretation than \citet*{sozou1998hyperbolic} (which produces a per-time-step of death via the hazard rate) and demonstrates that uncertainty over the \emph{timing} of rewards can also give rise to hyperbolic discounting and preference reversals, a hallmark of hyperbolic discounting.
However, though alternative motivations for hyperbolic discounting exist we build upon \citet*{sozou1998hyperbolic} for its clarity and simplicity.

Hyperbolic discounting was initially presumed to not lend itself to TD-based solutions \citep{daw2000behavioral} but the field has evolved on this point.  \citet*{maia2009reinforcement} proposes solution directions that find models that discount quasi-hyperbolically even though each learns with exponential discounting \citep{loewenstein1996out} but reaffirms the difficulty.  Finally, \citet*{alexander2010hyperbolically} proposes hyperbolically discounted temporal difference (HDTD) learning by making connections to hazard.  However, this approach introduces two additional free parameters to adjust for differences in reward-level. 

\textbf{Behavior RL and hyperbolic discounting in neuroscience.}
TD-learning has long been used for modeling behavioral reinforcement learning \citep{montague1996framework, schultz1997neural, sutton2018reinforcement}.
TD-learning computes the error as the difference between the expected value and actual value \citep{sutton2018reinforcement, daw2003reinforcement} where the error signal emerges from unexpected rewards.
However, these computations traditionally rely on exponential discounting as part of the estimate of the value which disagrees with empirical evidence in humans and animals \citep{strotz1955myopia, mazur1985probability, mazur1997choice, ainslie1975specious, ainslie1992picoeconomics}.
Hyperbolic discounting has been proposed as an alternative to exponential discounting though it has been debated as an accurate model \citep{kacelnik1997normative, frederick2002time}.
Naive modifications to TD-learning to discount hyperbolically present issues since the simple forms are inconsistent \citep{daw2000behavioral, redish2010neural} RL models have been proposed to explain behavioral effects of humans and animals \citep{fu2006recurrent, rangel2008framework} but \cite{kurth2009temporal} demonstrated that distributed exponential discount factors can directly model hyperbolic discounting.  
This work proposes the $\mu$Agent, an agent that models the value function with a specific discount factor $\gamma$.  
When the distributed set of $\mu$Agent's votes on the action, this was shown to approximate hyperbolic discounting well in the adjusting-delay assay experiments \citep{mazur1987adjusting}.
Using the hazard formulation established in \cite{sozou1998hyperbolic}, we demonstrate how to extend this to other non-hyperbolic discount functions and demonstrate the efficacy of using a deep neural network to model the different Q-values from a shared representation.

\textbf{Towards more flexible discounting in reinforcement learning.}
RL researchers have recently adopted more flexible versions beyond a fixed discount factor \citep{feinberg1994markov, sutton1995td, sutton2011horde, white2017unifying}.
Optimal policies are studied in \cite{feinberg1994markov} where two value functions with different discount factors are used.
Introducing the discount factor as an argument to be queried for a set of timescales is considered in both Horde \citep{sutton2011horde} and $\gamma$-nets \citep{sherstan2018generalizing}.
\cite{reinke2017average} proposes the Average Reward Independent Gamma Ensemble framework which imitates the average return estimator.

\citet*{lattimore2011time} generalizes the original discounting model through discount functions that vary with the age of the agent, expressing time-inconsistent preferences as in hyperbolic discounting.
The need to increase training stability via effective horizon was addressed in \citet*{franccois2015discount} who proposed dynamic strategies for the discount factor $\gamma$.
Meta-learning approaches to deal with the discount factor have been proposed in \citet*{xu2018meta}.
Finally, \citet*{pitis2019rethinking} characterizes rational decision making in sequential processes, formalizing a process that admits a state-action dependent discount rates.
This body of work suggests growing tension between the original MDP formulation with a fixed $\gamma$ and future research directions.

Operating over multiple time scales has a long history in RL.
\citet*{sutton1995td} generalizes the work of \citet*{singh1992scaling} and \citet*{dayan1993feudal} to formalize a multi-time scale TD learning model theory.
Previous work has been explored on solving MDPs with multiple reward functions and multiple discount factors though these relied on separate transition models \citep{feinberg1999constrained, dolgov2005stationary}.
\citet*{edwardsexpressing} considers decomposing a reward function into separate components each with its own discount factor. 
In our work, we continue to model the same rewards, but now model the value over different horizons.
Recent work in difficult exploration games demonstrates the efficacy of two different discount factors \citep{burda2018exploration} one for intrinsic rewards and one for extrinsic rewards.
Finally, and concurrent with this work, \cite{romoff2019separating} proposes the TD$(\Delta)$-algorithm which breaks a value function into a series of value functions with smaller discount factors.

\textbf{Auxiliary tasks in reinforcement learning.}
Finally, auxiliary tasks have been successfully employed and found to be of considerable benefit in RL. \citet*{suddarth1990rule} used auxiliary tasks to facilitate representation learning.
Building upon this, work in RL has consistently demonstrated benefits of auxiliary tasks to augment the low-information coming from the environment through extrinsic rewards
\citep{lample2017playing, mirowski2016learning}, \citep{jaderberg2016reinforcement, veeriah2018many, sutton2011horde}

% TODO(marc): Cite Lattimore (2012), Sutton et al. (2016) -- emphatic TD, White (2018) -- discount factors.

%%%%%%%%%%%%%%%%%%%%%%%%%%%%%%%%%%%%%%%%%%%%%%%%%%%%%%%%%%%%%
%              From Hazard to Discount Function
%%%%%%%%%%%%%%%%%%%%%%%%%%%%%%%%%%%%%%%%%%%%%%%%%%%%%%%%%%%%%
\section{Belief of Risk Implies a Discount Function}\label{sec: hazard_implies_discount}
\citet*{sozou1998hyperbolic} formalizes time preferences in which future rewards are discounted based on the probability that the agent will not \emph{survive} to collect them due to an encountered risk or \emph{hazard}. 

\begin{definition}
    \emph{Survival $s(t)$ is the probability of the agent surviving until time $t$.}
    \begin{equation}
        s(t) = P(\text{agent is alive} | \text{at time}\; t)
    \end{equation}
    
    \label{def: survival}
\end{definition}

A future reward $r_t$ is less valuable presently if the agent is unlikely to survive to collect it. 
If the agent is risk-neutral, the present value of a future reward $r_t$ received at time-$t$ should be discounted by the probability that the agent will survive until time $t$ to collect it, $s(t)$.\footnote{Note the difference in RL where future rewards are discounted by \emph{time-delay} so the value is $v(r_t) = \gamma^{t}r_t$.}
\begin{equation}
    v(r_t) = s(t)r_t
    \label{eqn: value_discount}
\end{equation}

Consequently, if the agent is certain to survive, $s(t)=1$, then the reward is not discounted per Equation \ref{eqn: value_discount}.  From this it is then convenient to define the hazard rate.

\begin{definition}
    \emph{Hazard rate $h(t)$ is the negative rate of change of the log-survival at time $t$}
    
    \begin{equation}
        h(t) = -\frac{d}{dt}  \text{ln}s(t) 
    \end{equation}
    \label{def: hazard}
\end{definition}

or equivalently expressed as $h(t) = -\frac{ds(t)}{dt} \frac{1}{s(t)}$.  
Therefore the environment is considered hazardous at time $t$ if the log survival is decreasing sharply.

\citet*{sozou1998hyperbolic} demonstrates that the prior belief of the risk in the environment implies a specific discounting function.
When the risk occurs at a known constant rate than the agent should discount future rewards exponentially.
% delta-prior over the \emph{hazard} in the environment.  
% That is, the agent holds a prior of the form $p(\lambda) = \delta(\lambda - \lambda_0)$ for some specific $\lambda_0$.
However, when the agent holds \emph{uncertainty} over the hazard rate then hyperbolic and alternative discounting rates arise.

\subsection{Known Hazard Implies Exponential Discount}\label{sec: single_hazard}
We recover the familiar exponential discount function in RL based on a prior assumption that the environment has a \emph{known constant} hazard.
Consider a known hazard rate of $h(t) = \lambda\ \geq 0$.
Definition \ref{def: hazard} sets a first order differential equation $\lambda =-\frac{d}{dt}  \text{ln}s(t) = -\frac{ds(t)}{dt} \frac{1}{s(t)}$.    
The solution for the survival rate is $s(t) = e^{-\lambda t}$ which can be related to the RL discount factor $\gamma$
\begin{equation}
    s(t) = e^{-\lambda t} = \gamma^t
\end{equation}
This interprets $\gamma$ as the per-time-step probability of the episode continuing.
This also allows us to connect the hazard rate $\lambda \in [0, \infty]$ to the discount factor $\gamma \in [0, 1)$.
\begin{equation}
\gamma = e^{-\lambda}
\label{eqn: gamma_to_lambda}
\end{equation}
As the hazard increases $\lambda \rightarrow \infty$, then the corresponding discount factor becomes increasingly myopic $\gamma \rightarrow 0$.  Conversely, as the environment hazard vanishes, $\lambda \rightarrow 0$, the corresponding agent becomes increasingly far-sighted $\gamma \rightarrow 1$.

In RL we commonly choose a single $\gamma$ which is consistent with the prior belief that there exists a known constant hazard rate $\lambda=-\text{ln}(\gamma)$.
We now relax the assumption that the agent holds this strong prior that it \emph{exactly} knows the true hazard rate.
From a Bayesian perspective, a looser prior allows for some uncertainty in the underlying hazard rate of the environment which we will see in the following section.

\subsection{Uncertain Hazard Implies Non-Exponential Discount}
% Our world is plagued with an uncertain uncertainty.  
We may not always be so confident of the true risk in the environment and instead reflect this underlying uncertainty in the hazard rate through a hazard prior $p(\lambda)$.  
Our survival rate is then computed by weighting specific exponential survival rates defined by a given $\lambda$ over our prior $p(\lambda)$
\begin{equation}
    s(t) = \int_{\lambda=0}^\infty  p(\lambda) e^{-\lambda t} d \lambda
    \label{eqn: survival_laplacian}
\end{equation}
\citet*{sozou1998hyperbolic} shows that under an exponential prior of hazard $p(\lambda) = \frac{1}{k} \text{exp}(-\lambda / k)$ the expected survival rate for the agent is \emph{hyperbolic}
\begin{equation}
    s(t) = \frac{1}{1 + k t} \equiv \Gamma_k(t)
    \label{eqn: hyp_discount}
\end{equation}
%and therefore we should discount future rewards by this.  
We denote the hyperbolic discount by $\Gamma_k(t)$ to make the connection to $\gamma$ in reinforcement learning explicit.
Further, \citet*{sozou1998hyperbolic} shows that different priors over hazard correspond to different discount functions.
We reproduce two figures in Figure \ref{fig: prior_to_discount} showing the correspondence between different hazard rate priors and the resultant discount functions.
The common approach in RL is to maintain a delta-hazard (black line) which leads to exponential discounting of future rewards.
Different priors lead to non-exponential discount functions.

\begin{figure}[h]
    \centering
    \includegraphics[width=0.9\columnwidth]{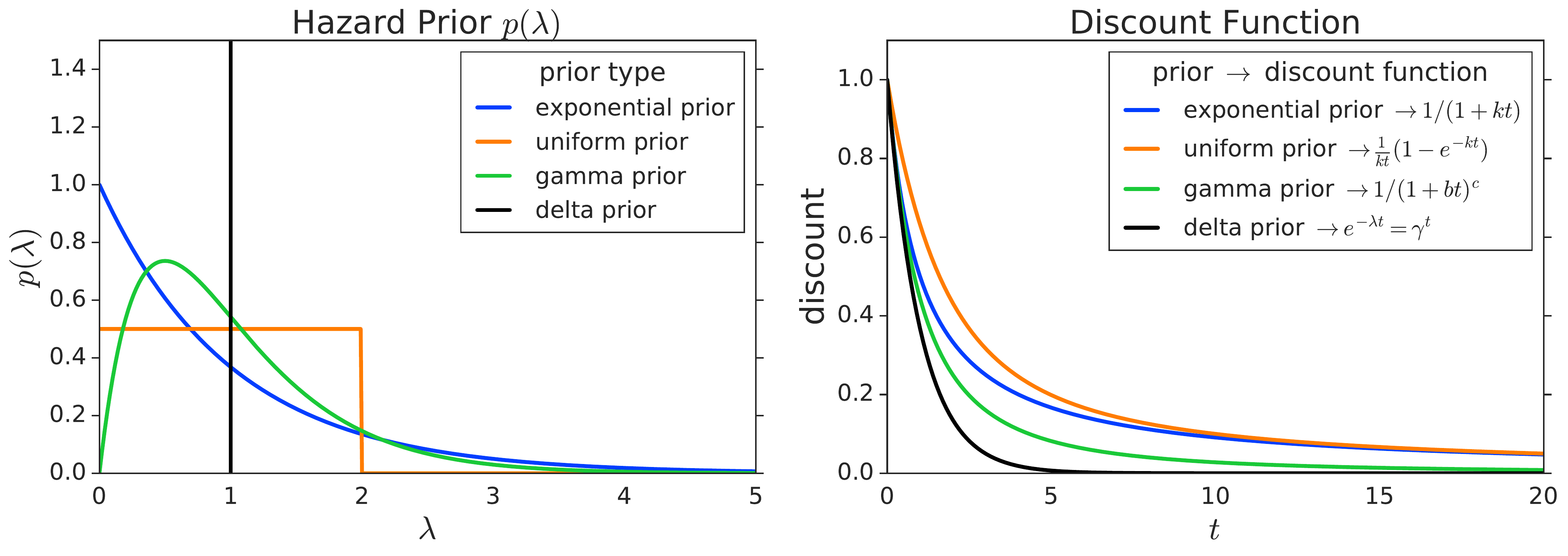}
    \caption{We reproduce two figures from \citet*{sozou1998hyperbolic}. There is a correspondence between hazard rate priors and the resulting discount function.  In RL, we typically discount future rewards exponentially which is consistent with a Dirac delta prior (black line) on the hazard rate indicating \emph{no uncertainty} of hazard rate.  However, this is a special case and priors with uncertainty over the hazard rate imply new discount functions. All priors have the same mean hazard rate $\mathbb{E}[p(\lambda)] = 1$.}
    \label{fig: prior_to_discount}
\end{figure}

\section{Hazard in MDPs}\label{sec: hazard_mdps}
To study MDPs with \emph{hazard distributions} and \emph{general discount functions} we introduce two modifications.
The hazardous MDP now is defined by the tuple $<\mathcal{S},\mathcal{A}, R, P, \mathcal{H}, d>$.
In standard form, the state space $\mathcal{S}$ and the action space $\mathcal{A}$ may be discrete or continuous.  
The learner observes samples from the environment transition probability $P(s_{t+1}|s_t, a_t)$ for going from $s_{t}\in \mathcal{S}$ to $s_{t+1}\in \mathcal{S}$ given $a_t \in \mathcal{A}$. We will consider the case where $P$ is a sub-stochastic transition function, which defines an episodic MDP.
The environment emits a bounded reward $r: \mathcal{S} \times \mathcal{A} \rightarrow \left[r_{min},  r_{max}\right]$ on each transition.
In this work we consider non-infinite episodic MDPs.

The first difference is that at the beginning of each episode, a hazard $\lambda \in [0, \infty)$ is sampled from the hazard distribution $\mathcal{H}$. 
This is equivalent to sampling a \emph{continuing} probability $\gamma =  e^{-\lambda}$. 
During the episode, the hazard modified transition function will be $P_\lambda$, in that $P_\lambda(s' | s, a) = e^{-\lambda} P(s' | s, a)$. 
The second difference is that we now consider a general discount function $d(t)$.  This differs from the standard approach of exponential discounting in RL with $\gamma$ according to $d(t)=\gamma^t$, which is a special case.

This setting makes a close connection to partially observable Markov Decision Process (POMDP) \citep{kaelbling1998planning} where one might consider $\lambda$ as an unobserved variable. 
However, the classic POMDP definition contains an explicit discount function $\gamma$ as part of it's definition which does not appear here.

A policy $\pi: \mathcal{S} \rightarrow \mathcal{A}$ is a mapping from states to actions. 
The state action value function $Q_\pi^{\mathcal{H}, d}(s,a)$ is the expected discounted rewards after taking action $a$ in state $s$ and then following policy $\pi$ until termination.
 \begin{equation}
     Q^{\mathcal{H}, d}_{\pi}(s, a) = \mathbb{E}_{\lambda}  \mathbb{E}_{ \pi, P_\lambda} \left[ \sum_{t=0}^\infty d(t) R(s_t, a_t) | s_0 = s, a_0 = a \right]
 \end{equation}
where $\lambda \sim \mathcal{H}$ and $\mathbb{E}_{ \pi, P_\lambda}$ implies that $s_{t+1} \sim  P_\lambda(\cdot | s_t, a_t)$ and $a_t \sim \pi(\cdot | s_t)$.

\subsection{Equivalence Between Hazard and Discounting}
In the hazardous MDP setting we observe the same connections between hazard and discount functions delineated in Section \ref{sec: hazard_implies_discount}. 
This expresses an equivalence between the value function of an MDP with a discount and MDP with a hazard distribution. 

For example, there exists an equivalence between the exponential discount function $d(t) = \gamma^t$ to the \emph{undiscounted} case where the agent is subject to a  $(1-\gamma)$ per time-step of dying \citep{lattimore2011time}.
The typical Q-value (left side of Equation \ref{eqn: discount_hazard_equivalence}) is when the agent acts in an environment without hazard $\lambda = 0$ or $\mathcal{H} = \delta(0)$ and discounts future rewards according to $d(t) = \gamma^t = e^{-\lambda t}$ which we denote as $Q_\pi^{\delta(0), \gamma^t} (s,a)$.  
The alternative Q-value (right side of Equation \ref{eqn: discount_hazard_equivalence}) is when the agent acts under hazard rate $\lambda = -\ln \gamma$ but does not discount future rewards which we denote as $Q_\pi^{\delta(-\ln \gamma), 1} (s,a)$.
\begin{equation}\label{eqn: discount_hazard_equivalence}
Q_\pi^{\delta(0), \gamma^t} (s,a) = Q_\pi^{\delta(-\ln \gamma), 1} (s,a)\; \forall\; \pi, s, a.
\end{equation}
where $\delta(x)$ denotes the Dirac delta distribution at $x$. 
This follows from $P_\lambda(s' | s, a) = e^{-\lambda} P(s' | s, a)$
\begin{align*}
    \mathbb{E}_{ \pi, P} \left[ \sum_{t=0}^\infty \gamma^t R(s_t, a_t) |  s_0 = s, a_0 = a \right] &= \mathbb{E}_{ \pi, P} \left[ \sum_{t=0}^\infty e^{-\lambda t} R(s_t, a_t) | s_0 = s, a_0 = a \right] \\
    &= \mathbb{E}_{ \pi, P_\lambda} \left[ \sum_{t=0}^\infty R(s_t, a_t) | s_0 = s, a_0 = a \right] \\
\end{align*}
Following Section \ref{sec: hazard_implies_discount} we also show a similar equivalence between hyperbolic discounting and the specific hazard distribution
$p_k(\lambda) = \frac{1}{k} \text{exp} (-\lambda / k)$, where again, $\lambda \in [0, \infty)$ in Appendix \ref{appendix:  hyperbolic_equiv_hazard}.

\begin{equation*}
Q_\pi^{\delta(0), \Gamma_k} (s,a) = Q_\pi^{p_k, 1} (s,a) 
\end{equation*}

For notational brevity later in the paper, we will omit the explicit hazard distribution $\mathcal{H}$-superscript if the environment is not hazardous.

%%%%%%%%%%%%%%%%%%%%%%%%%%%%%%%%%%%%%%%%%%%%%%%%%%%%%%%%%%%%%%%%%%%%%%%%%%%%%%%%%%
% Computing Hyperbolic discount functions Using Exponential Discount Functions $\gamma$
%%%%%%%%%%%%%%%%%%%%%%%%%%%%%%%%%%%%%%%%%%%%%%%%%%%%%%%%%%%%%%%%%%%%%%%%%%%%%%%%%%
\section{Computing Hyperbolic Q-Values From Exponential Q-Values}\label{sec: est_discount_through_gamma}
We show how one can re-purpose exponentially-discounted Q-values to compute hyperbolic (and other-non-exponential) discounted Q-values.
% Section \ref{sec: hazard_implies_discount} recounts the connection of a hazard rate prior $p(\lambda)$ to a specific discount function and Section \ref{sec: hazard_mdps} describes the connection to MDPs.
% We show here that these new discount functions can be computed using using an infinite set of standard exponential Q-values learned in RL.
The central challenge with using non-exponential discount strategies is that most RL algorithms use some form of TD learning \citep{sutton1988learning}. 
This family of algorithms exploits the Bellman equation \citep{bellman1958routing} which, when using exponential discounting, relates the value function at one state with the value at the following state.
\begin{equation}
Q_\pi^{\gamma^t}(s,a) = \mathbb{E}_{\pi, P} [R(s,a) + \gamma Q_\pi(s', a')]
\end{equation}
where expectation $\mathbb{E}_{\pi, P}$ denotes sampling $a \sim \pi(\cdot| s)$, $s' \sim P(\cdot | s, a)$, and $a' \sim \pi(\cdot | s')$.

Being able to reuse the literature on TD methods without being constrained to exponential discounting is thus an important challenge. 
\subsection{Computing Hyperbolic $Q$-Values}
Let's start with the case where we would like to estimate the value function where rewards are discounted hyperbolically instead of the common exponential scheme.
We refer to the hyperbolic Q-values as $Q^{\Gamma}_{\pi}$ below in Equation \ref{eqn: non_exponential_discount}
\begin{align}
    Q^{\Gamma_k}_{\pi}(s, a) =& \mathbb{E}_{\pi} \left[\Gamma_k(1)R(s_1, a_1) + \Gamma_k(2)R(s_2, a_2) + \cdots \biggr| s, a \right]\\
                    =& \mathbb{E}_{\pi} \left[ \sum_t \Gamma_k(t)R(s_t, a_t)\biggr| s, a \right]
    \label{eqn: non_exponential_discount}
\end{align}

We may relate the hyperbolic $Q^{\Gamma}_{\pi}$-value to the values learned through standard $Q$-learning.  
To do so, notice that the hyperbolic discount $\Gamma_t$ can be expressed as the integral of a certain function $f(\gamma, t)$ for $\gamma=[0, 1)$ in Equation \ref{eqn: hyperbolic_original}.
\begin{equation}
    \int_{\gamma=0}^1 \gamma^{kt} d\gamma = \frac{1}{1 + kt} = \Gamma_k(t) 
    \label{eqn: hyperbolic_original}
\end{equation}

The integral over this specific function $f(\gamma, t) = \gamma^{kt}$ yields the desired hyperbolic discount factor $\Gamma_k(t)$ by considering an \emph{infinite set} of exponential discount factors $\gamma$ over its domain $\gamma \in [0, 1)$.
We visualize the hyperbolic discount factors $\frac{1}{1+t}$ (consider $k=1$) for the first few time-steps $t$ in Figure \ref{fig: exact_integral}.

\begin{figure*}[h]
    \centering
    \includegraphics[width=\textwidth]{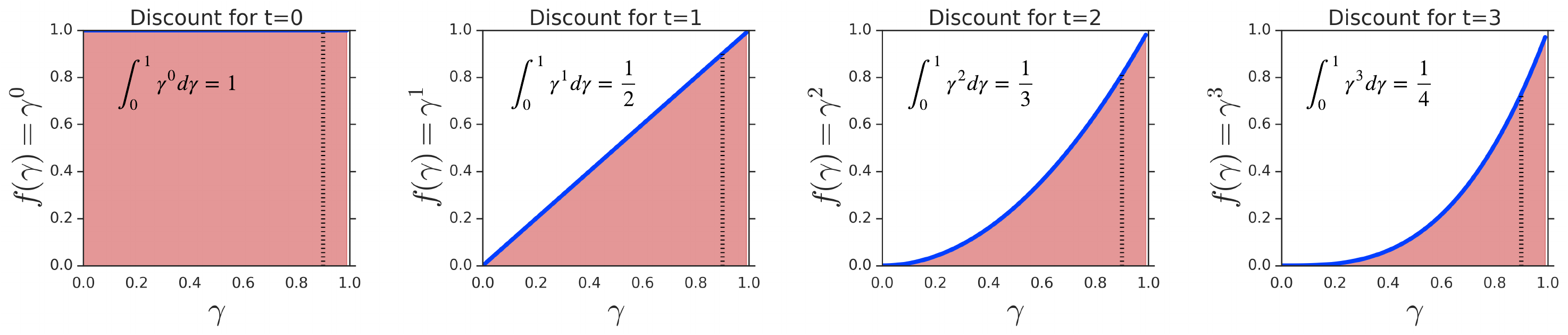}
    \caption{From left to right we consider the first four time-steps ($t=0, 1, 2, 3$) of the function $\gamma^t$ (shown in blue) over the valid range. The integral (red) of $\gamma^t$ at time $t$ equals the \emph{hyperbolic} discount function $1/(1 + t)$ shown in each subplot.  Time $t=0$ is not discounted since the integral of $\gamma^0=1$ from 0 to 1 is 1.  Then $t=1$ is discounted by $\frac{1}{2}$, $t=2$ is discounted by $t=\frac{1}{3}$ and so on.  For illustration, the black dotted vertical line indicates the discount that we would use for each time-step if we considered only a single discount factor $\gamma=0.9$.}
    \label{fig: exact_integral}
\end{figure*}

Recognize that the integrand $\gamma^{kt}$ is the standard exponential discount factor which suggests a connection to standard Q-learning \citep{watkins1992q}.
This suggests that if we could consider an infinite set of $\gamma$ then we can combine them to yield hyperbolic discounts for the corresponding time-step $t$.
We build on this idea of modeling many $\gamma$ throughout this work.

We employ Equation \ref{eqn: hyperbolic_original} and return to the task of computing hyperbolic Q-values $Q^{\Gamma}_{\pi}(s, a)$\footnote{Hyperbolic Q-values can generally be infinite for bounded rewards. We consider non-infinite episodic MDPs only.}
\begin{align}
    Q^{\Gamma}_{\pi}(s, a) 
                    =& \mathbb{E}_{\pi} \left[ \sum_t \Gamma_k(t)R(s_t, a_t)\biggr| s, a \right] \\
                    =& \mathbb{E}_{\pi} \left[ \sum_t \left( \int_{\gamma=0}^1  \gamma^{kt} d\gamma \right) R(s_t, a_t)  \biggr| s, a \right] \\
                    =& \int_{\gamma=0}^1 \mathbb{E}_{\pi} \left[ \sum_t R(s_t, a_t) (\gamma^k)^t  \biggr| s, a \right] d \gamma \\
                    =& \int_{\gamma=0}^1 Q^{(\gamma^k)^t}_{\pi}(s,a) d\gamma
    \label{eqn: hyperbolic_q_values}
\end{align}
where $\Gamma_k(t)$ has been replaced on the first line by $\left( \int_{\gamma=0}^1  \gamma^{kt} d\gamma\right)$ and the exchange is valid if $\sum_{t=0}^\infty \gamma^{kt} r_t < \infty$.
This shows us that we can compute the $Q^{\Gamma}_{\pi}$-value according to hyperbolic discount factor by considering an infinite set of $Q^{\gamma^k}_{\pi}$-values computed through standard $Q$-learning.
Examining further, each $\gamma \in [0, 1)$ results in TD-errors learned for a new $\gamma^k$.
For values of $k<1$, which extends the horizon of the hyperbolic discounting, this would result in larger $\gamma$.

\subsection{Generalizing to Other Non-Exponential $Q$-Values}
Equation \ref{eqn: hyperbolic_original} computes hyperbolic discount functions but its origin was not mathematically motivated.
We consider here an alternative scheme to deduce ways to model hyperbolic as well as different discount schemes through integrals of $\gamma$.

\begin{lemma}\label{lem: exponential_discount_weighting}
\emph{Let $Q_\pi^{\mathcal{H}, \gamma}(s,a)$ be the state action value function under exponential discounting in a hazardous MDP $<\mathcal{S}, \mathcal{A}, R, P, \mathcal{H}, \gamma^t>$ and let $Q_\pi^{\mathcal{H}, d}(s,a)$ refer to the value function in the same MDP except for new discounting $<\mathcal{S}, \mathcal{A}, R, P, \mathcal{H}, d>$. 
If there exists a function $w: [0, 1] \to \mathbb{R}$ such that
\begin{equation}
d(t) = \int_0^1 w(\gamma) \gamma^t d\gamma
\end{equation}
which we will refer to as the exponential weighting condition, then
\begin{equation}
Q_\pi^{\mathcal{H}, d}(s,a) = \int_0^1 w(\gamma) Q_\pi^{\mathcal{H}, \gamma}(s,a) d\gamma
\end{equation}
}
\end{lemma}
\begin{proof}
Applying the condition on $d$,
    \begin{align}
    Q_\pi^{\mathcal{H}, d}(s,a)  &= \mathbb{E}_{\lambda}  \mathbb{E}_{ \pi, P_\lambda} \left[ \sum_{t=0}^\infty \left( \int_0^1 w(\gamma) \gamma^t  d\gamma \right) R(s_t, a_t) | s_0 = s, a_0 = a \right] \\
    &= \int_0^1 \mathbb{E}_{\lambda}  \mathbb{E}_{ \pi, P_\lambda} w(\gamma) \left[ \sum_{t=0}^\infty  \gamma^t R(s_t, a_t) | s_0 = s, a_0 = a \right] d\gamma \\
    &= \int_0^1 w(\gamma) Q_\pi^{\mathcal{H}, \gamma}(s,a) d\gamma              
    \end{align}
\end{proof}
where again the exchange is valid if $\sum_{t=0}^\infty \gamma^t R(s_t, a_t) < \infty$. 
We can now see that the exponential weighting condition is satisfied for hyperbolic discounting and a list of other discounting that we might want to consider.

For instance, the hyperbolic discount can also be expressed as the integral of a different function $f(\gamma, t)$ for $\gamma=[0, 1)$ in Equation \ref{eqn: hyperbolic}.
\begin{equation}
    \frac{1}{k} \int_{\gamma=0}^1  \gamma ^{1/k + t-1}  d\gamma = \frac{1}{1 + kt}
    \label{eqn: hyperbolic}
\end{equation}
As before, an integral over a function $f'(\gamma, t) = \frac{1}{k}\gamma^{1/k + t - 1} = w(\gamma)\gamma^t$ yields the desired hyperbolic discount factor $\Gamma_k(t)$. 
This integral can be derived by recognizing Equation \ref{eqn: survival_laplacian} as the Laplace transform of the prior $\mathcal{H}=p(\lambda)$ and then applying a change of variables $\gamma=e^{-\lambda}$.  Computing hyperbolic and other discount functions is demonstrated in detail in Appendix \ref{appendix: alternative_discount}.
We summarize in Table \ref{tbl: alternative_discounts} how a particular hazard prior $p(\lambda)$ can be computed via integrating over specific weightings $w(\gamma)$ and the corresponding discount function.

\begin{table*}[h!]
    \centering
    \bgroup
    % \footnotesize
    \def\arraystretch{2}
        \begin{tabular}{c | c c c}
        & $\mathcal{H}=p(\lambda)$ &  $d(t)$ & $w(\gamma)$\\
        \hline \hline
        Dirac Delta Prior & $\delta(\lambda - k)$ & $e^{-kt} (= (\gamma_k)^t)$ & $\frac{1}{\gamma}\delta(-\ln \gamma - k)$  \\
        Exponential Prior & $\frac{1}{k}e^{-\lambda / k}$ & $\frac{1}{1 + kt}$ & $\frac{1}{k} \gamma ^{1/k -1}$ \\
        Uniform Prior & $  
            \begin{cases}
              \frac{1}{k}, & \text{if}\ \lambda \in [0, k] \\
              0, & \text{otherwise}
            \end{cases}$
            & $\frac{1}{kt} \left(1 - e^{-kt} \right)$ 
            & $  
            \begin{cases}
                \frac{1}{k}\gamma^{-1}, & \text{if}\ \gamma \in [e^{-k}, 1] \\
                0, & \text{otherwise}
            \end{cases}$\\
        \end{tabular}
    \egroup
    \caption{Different hazard priors $\mathcal{H} = p(\lambda)$ can be alternatively expressed through weighting exponential discount functions $\gamma^t$ by $w(\gamma)$. This table matches different hazard distributions to their associated discounting function and the weighting function per Lemma \ref{lem: exponential_discount_weighting}. The typical case in RL is a Dirac Delta Prior over hazard rate $\delta(\lambda - k)$. We only show this in detail for completeness; one would not follow such a convoluted path to arrive back at an exponential discount but this approach holds for richer priors.
    The derivations can be found in the Appendix \ref{appendix: alternative_discount}.}
    \label{tbl: alternative_discounts}
\end{table*}

% \begin{table*}[h!]
%     \centering
%     \bgroup
%     % \footnotesize
%     \def\arraystretch{2}
%         \begin{tabular}{c | c c c}
%         Hazard Prior & $p(\lambda)$ & Integral over $\gamma$ & Discount Rate\\
%         \hline \hline 
%         Exponential prior& $\frac{1}{k}e^{-\lambda / k}$ & $\int_{\gamma=0}^1  \gamma ^{kt} d\gamma$ & $\frac{1}{1 + kt}$ \\
%         Exponential prior (\emph{alternative}) & $\frac{1}{k}e^{-\lambda / k}$ & $\frac{1}{k} \int_{\gamma=0}^1  \gamma ^{1/k + t-1} d\gamma$ & $\frac{1}{1 + kt}$ \\
%         Uniform prior & $\frac{1}{k}$, $\lambda \in [0, k]$ & $\frac{1}{k} \int^{1}_{\gamma=e^{-k}} \gamma^{t-1} d\gamma$ & $\frac{1}{kt} \left(1 - e^{-kt} \right)$\\
%         Delta prior & $\delta(\lambda - \lambda_0)$ & $\int_{\gamma=0}^1  \delta(-\text{ln}\gamma - \lambda_0)\gamma ^{t-1} d\gamma$ &  $\gamma_0^{t-1}$ \\
%         \end{tabular}
%     \egroup
%     \caption{Different hazard priors $p(\lambda)$ can be alternatively expressed as an integral over $\gamma$ and the corresponding discount function.  Through our work we consider the first entry in the table, that is, our original calculation for the hyperbolic discount, however, other options are possible.  The typical case is when we hold a delta prior over hazard $\delta(\lambda - \lambda_0)$ which is shown in the table as the integral over a delta-function yielding our typical discount function $\gamma_0^{t-1}$ for some chosen $\gamma_0$. Calculation details available in Appendix \ref{appendix: alternative_discount}.}
%     \label{tbl: alternative_discounts_other}
% \end{table*}

\section{Approximating Hyperbolic $Q$-Values}\label{sec: approx_hyperbolic_qvals}
Section \ref{sec: est_discount_through_gamma} describes an equivalence between hyperbolically-discounted Q-values and integrals of exponentially-discounted Q-values requiring evaluating an infinite set of value functions. 
We now present a \emph{practical} approach to approximate discounting $\Gamma(t) = \frac{1}{1+kt}$ using standard $Q$-learning \citep{watkins1992q}.

\subsection{Approximating the Discount Factor Integral}\label{sec: approx_integral}
To avoid estimating an infinite number of $Q^{\gamma}_{\pi}$-values we introduce a free hyperparameter ($n_{\gamma}$) which is the total number of $Q^{\gamma}_{\pi}$-values to consider, each with their own $\gamma$. 
We use a practically-minded approach to choose $\mathcal{G}$ that emphasizes evaluating larger values of $\gamma$ rather than uniformly choosing points and empirically performs well as seen in Section \ref{sec: hyperbolic_results}.
\begin{equation}
    \mathcal{G}= [\gamma_0, \gamma_1, \cdots, \gamma_{n_\gamma}]
\end{equation}
Our approach is described in Appendix \ref{appendix: gamma_interval}.
Each $Q_{\pi}^{\gamma_i}$ computes the discounted sum of returns according to that specific discount factor $Q^{\gamma_i}_{\pi}(s, a) =  \mathbb{E}_{\pi} \left[ \sum_t (\gamma_i)^t r_t | s_0 = s, a_0 = a \right]$.

We previously proposed two equivalent approaches for computing hyperbolic Q-values, but for simplicity we consider the one presented in Lemma \ref{lem: exponential_discount_weighting}.
The set of $Q$-values permits us to estimate the integral through a Riemann sum (Equation \ref{eqn: riemann_sum}) which is described in further detail in Appendix \ref{appendix:  hyperbolic_approach}.

\begin{align}\label{eqn: riemann_sum}
    Q_{\pi}^{\Gamma}(s,a) =& \int_0^1 w(\gamma) Q_\pi^{\gamma}(s,a) d\gamma \\
                          \approx& \sum_{\gamma_i \in \mathcal{G}} (\gamma_{i+1} - \gamma_i)\;w(\gamma_i)\;Q_\pi^{\gamma_i}(s,a)
\end{align}

where we estimate the integral through a lower bound.
We consolidate this entire process in Figure \ref{fig: model_description} where we show the full process of rewriting the hyperbolic discount rate, hyperbolically-discounted Q-value, the approximation and the instantiated agent.
This approach is similar to that of \cite{kurth2009temporal} where each $\mu$Agent models a specific discount factor $\gamma$.
However, this differs in that our final agent computes a weighted average over each Q-value rather than a sampling operation of each agent based on a $\gamma$-distribution.

\begin{figure}[!h]
    \centering
    \includegraphics[width=\columnwidth]{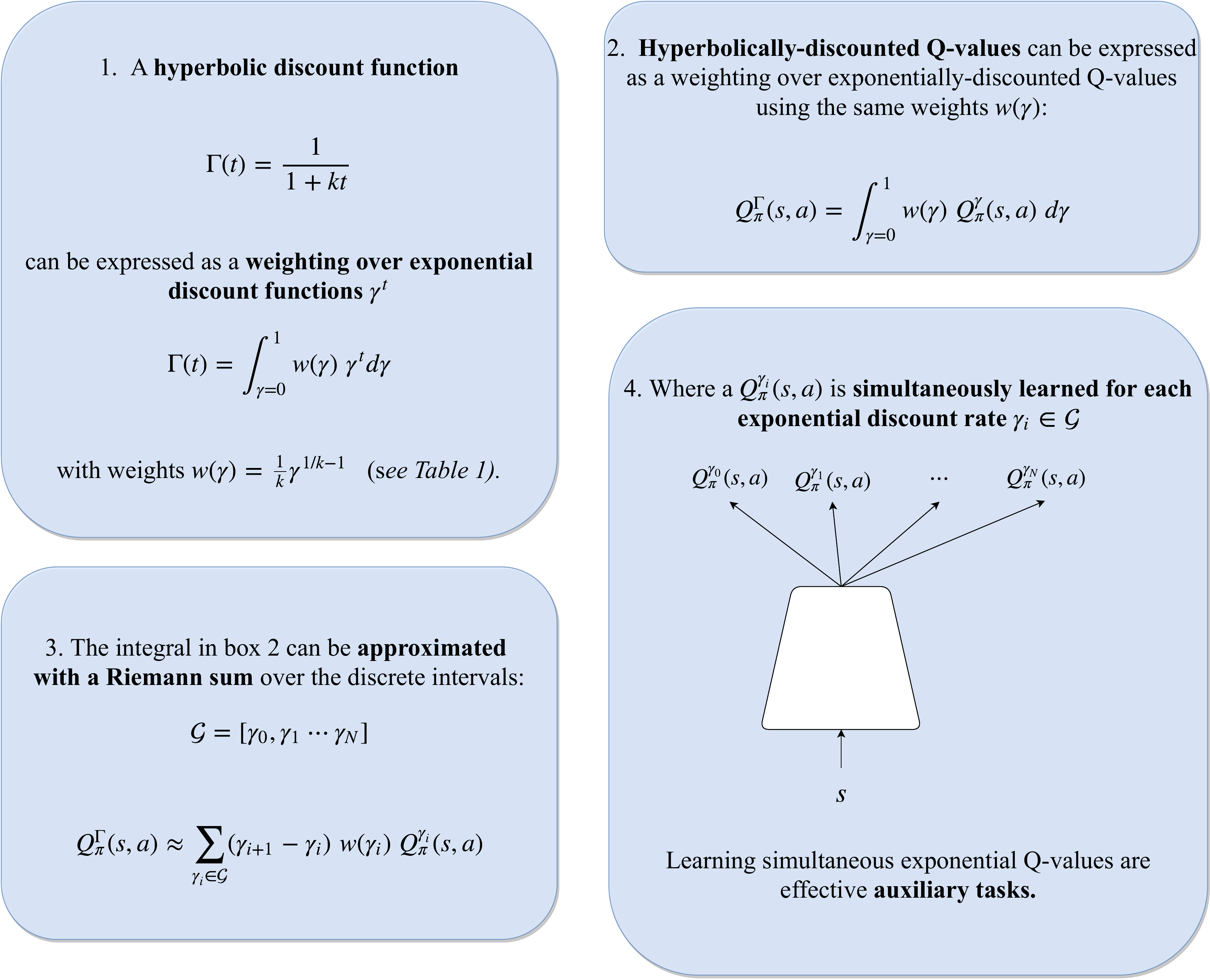}
    \caption{Summary of our approach to approximating hyperbolic (and other non-exponential) Q-values via a weighted sum of exponentially-discounted Q-vaulues.}
    \label{fig: model_description}
\end{figure}

% Here we diagram the case of $k=1$ for notation simplicity, but for $k\neq 1$ this would change simply by replacing $Q_\pi^\gamma$ with $Q_\pi^{\gamma^k}$.
% The parallel estimates for multiple $Q_\pi^\gamma$-values can be combined to approximate a hyperbolic $Q$-value using Equation \ref{eqn: hyperbolic_q_values}.

%%%%%%%%%%%
% Results
%%%%%%%%%%%

\section{Pathworld Experiments}\label{sec: hyperbolic_results}

\subsection{When to Discount Hyperbolically?}
The benefits of hyperbolic discounting will be greatest under:

\begin{enumerate}
    \item \textbf{Uncertain hazard}. The hazard-rate characterizing the environment is not known. For instance, an unobserved hazard-rate variable $\lambda$ is drawn independently at the beginning of each episode from $\mathcal{H} = p(\lambda)$.  
    \item \textbf{Non-trivial intertemporal decisions}. The agent faces \emph{non-trivial} intertemporal decision.  A non-trivial decision is one between smaller nearby rewards versus larger distant rewards.\footnote{A trivial intertemporal decision is one between small distant rewards versus large close rewards}.
\end{enumerate}

In the absence of both properties we would not expect any advantage to discounting hyperbolically.  
As described before, if there is a single-true hazard rate $\lambda_{\text{env}}$, than an optimal $\gamma^* = e^{-\lambda_\text{env}}$ exists and future rewards should be discounted exponentially according to it.  
Further, without \emph{non-trivial intertemporal trade-offs} which would occur if there is one path through the environment with perfect alignment of short- and long-term objectives, all discounting schemes will yield the same optimal policy.

\subsection{Pathworld Details}
We note two sources for discounting rewards in the future:  \emph{time delay} and \emph{survival probability} (Section \ref{sec: hazard_mdps}).  
% Section \ref{sec: hazard_implies_discount} notes at the incongruity between the discounted value implied through hazard versus the one implied through RL.  
% In RL, a future reward received at time $t$ is traditionally discounted based on time-delay $\gamma^{t}$ whereas in the hazard formulation a reward is discounted by the survival probability $s(t)$.  
% Therefore, a reward received with certainty in the future at time $t$ with $s(t)=1$ is not discounted under hazard but will be discounted in RL based on the time-delay.
In Pathworld of \ref{fig: hyperbolic_gridworld}, we train to maximize hyperbolically discounted returns ($\sum_t \Gamma_k(t) R(s_t, a_t)$) under no hazard ($\mathcal{H} = \delta(\lambda - 0)$) but then evaluate the undiscounted returns $d(t)=1.0 \; \forall \; t$ with the paths subject to hazard $\mathcal{H} = \frac{1}{k}\text{exp}(-\lambda/k)$.
Through this procedure, we are able to train an agent that is \emph{robust} to hazards in the environment.

The agent makes one decision in Pathworld (Figure \ref{fig: hyperbolic_gridworld}): which of the $N$ paths to investigate.  
Once a path is chosen, the agent continues until it reaches the end or until it dies. This is similar to a multi-armed bandit, with each action subject to dynamic risk.
The paths vary quadratically in length with the index $d(i) = i^2$ but the rewards increase linearly with the path index $r(i) = i$.  
This presents a non-trivial decision for the agent.  
At deployment, an unobserved hazard $\lambda \sim \mathcal{H}$ is drawn and the agent is subject to a per-time-step risk of dying of $(1-e^{-\lambda})$.  
This environment differs from the adjusting-delay procedure presented by \cite{mazur1987adjusting} and then later modified by \cite{kurth2009temporal}.
Rather then determining time-preferences through varaible-timing of rewards, we determine time-preferences through risk to the reward.
% There exist larger rewards further away, however, the paths are longer and the agent incurs a per-time-step risk of death.  

\begin{figure}
    \centering
    \includegraphics[width=0.45\columnwidth]{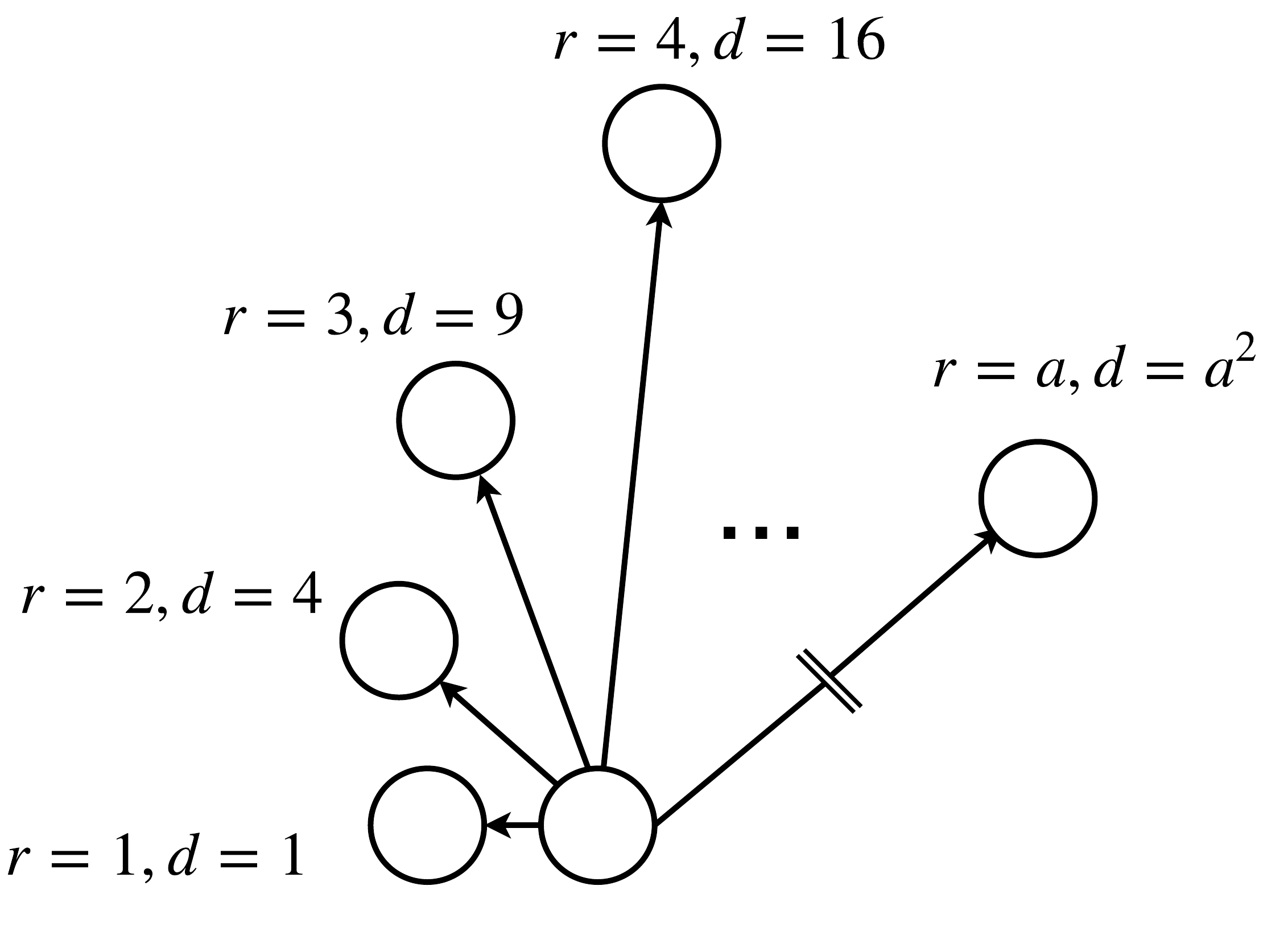}
    \caption{The Pathworld.  Each state (white circle) indicates the accompanying reward $r$ and the distance from the starting state $d$.  From the start state, the agent makes a single action: which which path to follow to the end.  Longer paths have a larger rewards at the end, but the agent incurs a higher risk on a longer path.}
    \label{fig: hyperbolic_gridworld}
\end{figure}

\subsection{Results in Pathworld}
% Our approach builds off many Q-values to approximate hyperbolic Q-values.
Figure \ref{fig: path_returns} validates that our approach well-approximates the true hyperbolic value of each path when the hazard prior matches the true distribution.  
Agents that discount exponentially according to a single $\gamma$ (as is commonly the case in RL) incorrectly value the paths.

\begin{figure}[h]
    \begin{floatrow}
        % Figure.
        \ffigbox[0.55\FBwidth]{%
            \includegraphics[width=\columnwidth]{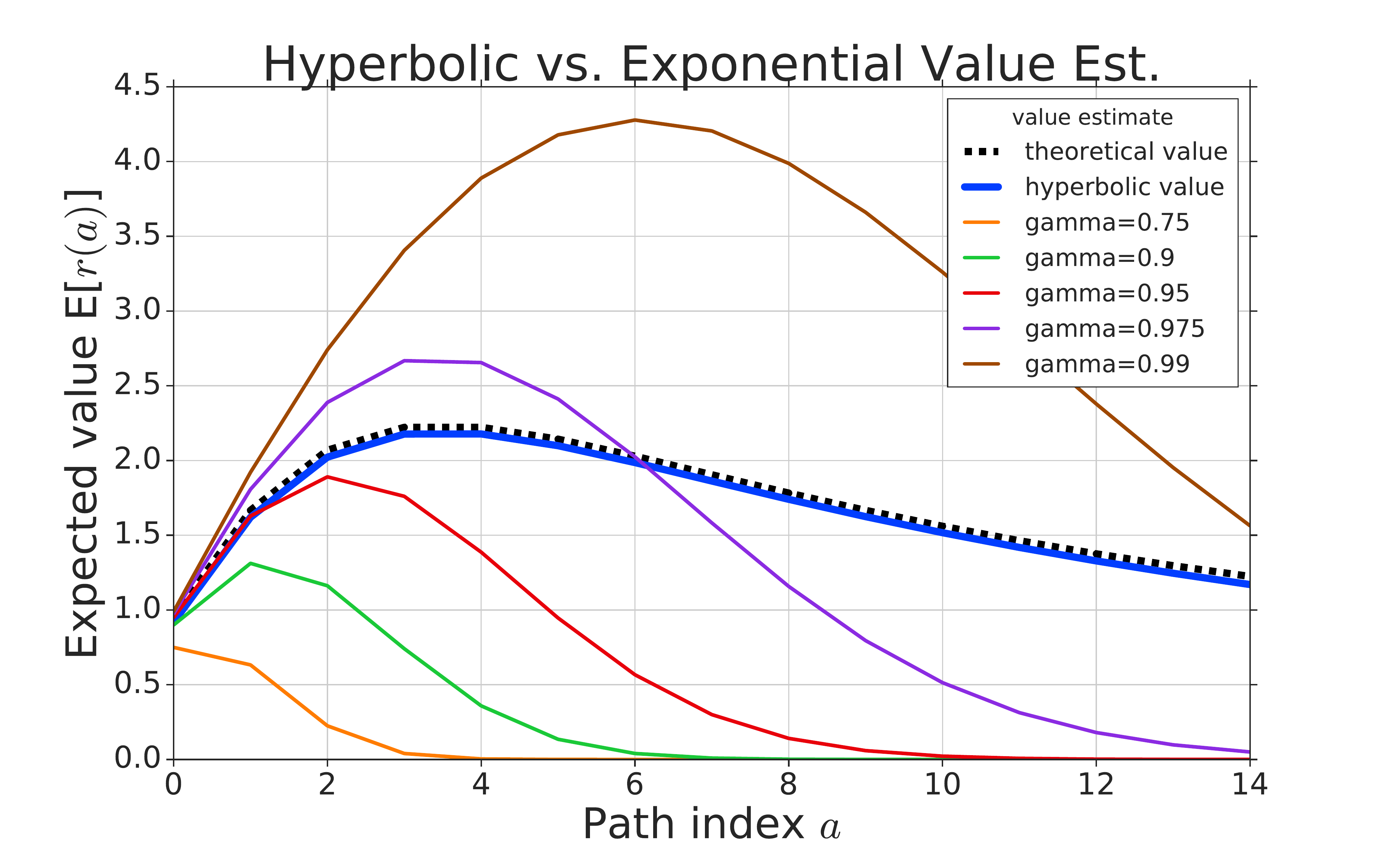}
            \caption{In each episode of Pathworld an unobserved hazard $\lambda \sim p(\lambda)$ is drawn and the agent is subject to a total risk of the reward not being realized of $(1-e^{-\lambda})^{d(a)}$ where $d(a)$ is the path length.  When the agent's hazard prior matches the true hazard distribution, the value estimate agrees well with the theoretical value.  Exponential discounts for many $\gamma$ fail to well-approximate the true value as seen to the right in Table \ref{tbl: path_returns_MSE}.}
            \label{fig: path_returns}
        }
        
        % Table.
        \capbtabbox{%
            % \small
            \def\arraystretch{1.35}
                \begin{tabular}{ c c } 
                    \textbf{Discount function} & \textbf{MSE} \\ 
                    \hline 
                    \textbf{hyperbolic value} & \textbf{0.002} \\
                    $\gamma$=0.975& 0.566 \\
                    $\gamma$=0.95& 1.461 \\
                     $\gamma$=0.9& 2.253 \\
                    $\gamma$=0.99& 2.288 \\
                    $\gamma$=0.75& 2.809 \\
                    & \\
                    & \\
                \end{tabular}
            }{%
                \caption{The average mean squared error (MSE) over each of the paths in Figure \ref{fig: path_returns} showing that our approximation scheme well-approximates the true value-profile.}
                \label{tbl: path_returns_MSE}
            }
    \end{floatrow}
\end{figure}

We examine further the failure of exponential discounting in this hazardous setting.
For this environment, the true hazard parameter in the prior was $k=0.05$ (i.e. $\lambda \sim 20\text{exp}(-\lambda / 0.05)$).  
Therefore, at deployment, the agent must deal with dynamic levels of risk and faces a non-trivial decision of which path to follow.
Even if we tune an agent's $\gamma = 0.975$ such that it chooses the correct arg-max path, it still fails to capture the functional form (Figure \ref{fig: path_returns}) and it achieves a high error over all paths (Table \ref{tbl: path_returns_MSE}). 
If the arg-max action was not available or if the agent was proposed to evaluate non-trivial intertemporal decisions, it would act sub-optimally.

In the next two experiments we consider the more realistic case where the agent's prior over hazard \emph{does not} exactly match the environment true hazard rate.  
In Figure \ref{fig: diff_hyperbolic_coefficients} we consider the case that the agent still holds an exponential prior but has the wrong coefficient $k$ and in Figure \ref{fig: uniform_prior} we consider the case where the agent still holds an exponential prior but the true hazard is actually drawn from a uniform distribution with the same mean.

\begin{figure}[h]
    \begin{floatrow}
        % Figure.
        \ffigbox[0.55\FBwidth]{%
            \includegraphics[width=\columnwidth]{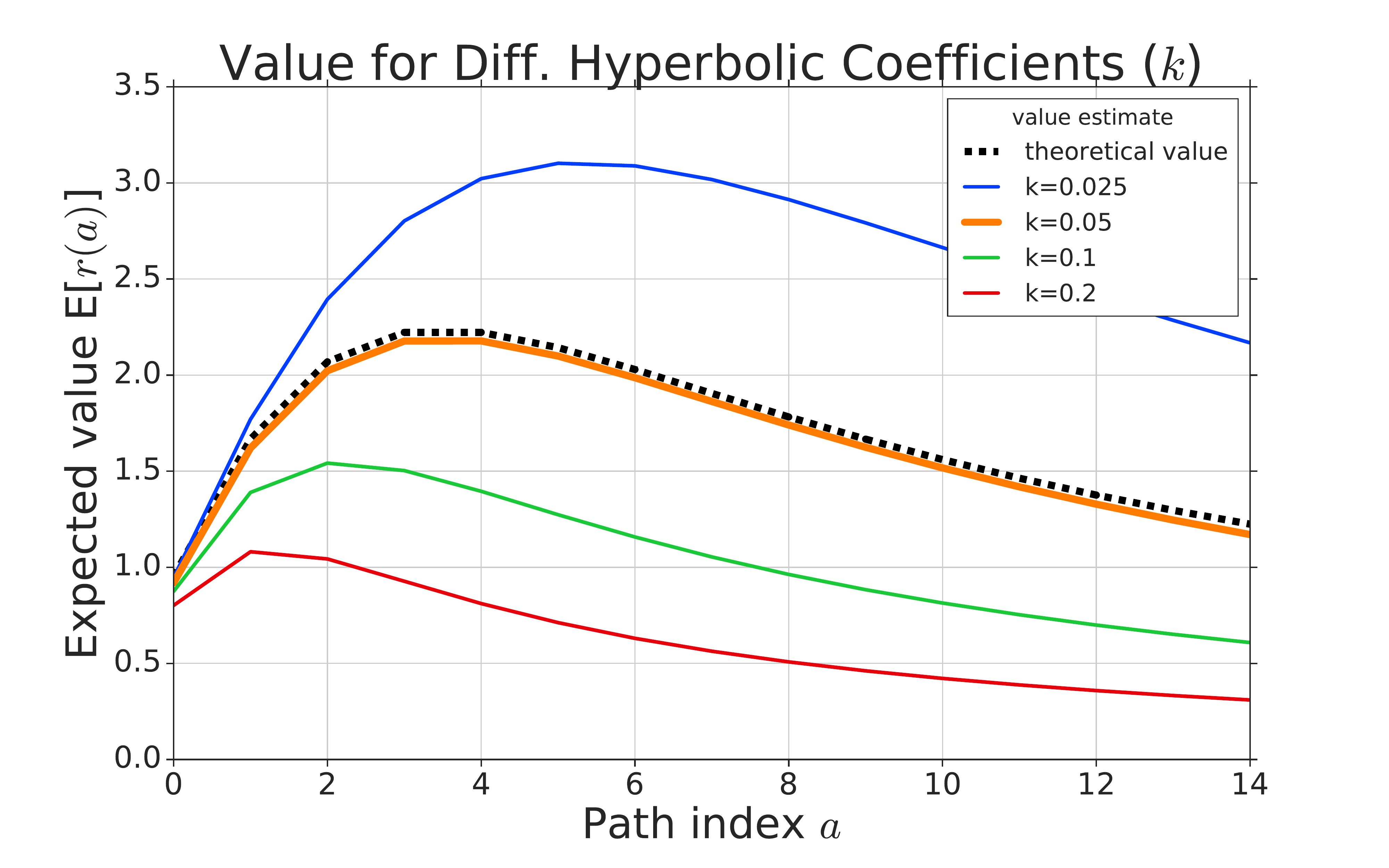}
            \caption{Case when the hazard coefficient $k$ \emph{does not} match that environment hazard.  Here the true hazard coefficient is $k=0.05$, but we compute values for hyperbolic agents with mismatched priors in range $k=[0.025, 0.05, 0.1, 0.2]$.  Predictably, the mismatched priors result in a higher prediction error of value but performs more reliably than exponential discounting, resulting in a cumulative lower error.  Numerical results in Table \ref{tbl: diff_hyperbolic_coefficients_MSE}.}
            \label{fig: diff_hyperbolic_coefficients}
        }
        
        % Table.
        \capbtabbox{%
            % \small
            \def\arraystretch{1.35}
                \begin{tabular}{ c c } 
                    \textbf{Discount function} & \textbf{MSE} \\ 
                    \hline 
                    k=0.05& \textbf{0.002} \\
                    k=0.1& 0.493 \\
                    k=0.025& 0.814 \\
                    k=0.2& 1.281 \\
                    & \\
                    & \\
                    & \\
                \end{tabular}
            }{%
                \caption{The average mean squared error (MSE) over each of the paths in Figure \ref{fig: diff_hyperbolic_coefficients}. As the prior is further away from the true value of $k=0.05$, the error increases.  However, notice that the errors for large factor-of-2 changes in $k$ result in generally lower errors than if the agent had considered only a single exponential discount factor $\gamma$ as in Table \ref{tbl: path_returns_MSE}.}
                \label{tbl: diff_hyperbolic_coefficients_MSE}
            }
    \end{floatrow}
\end{figure}

\begin{figure}[h]
    \begin{floatrow}
        % Figure.
        \ffigbox[0.55\FBwidth]{%
            \includegraphics[width=\columnwidth]{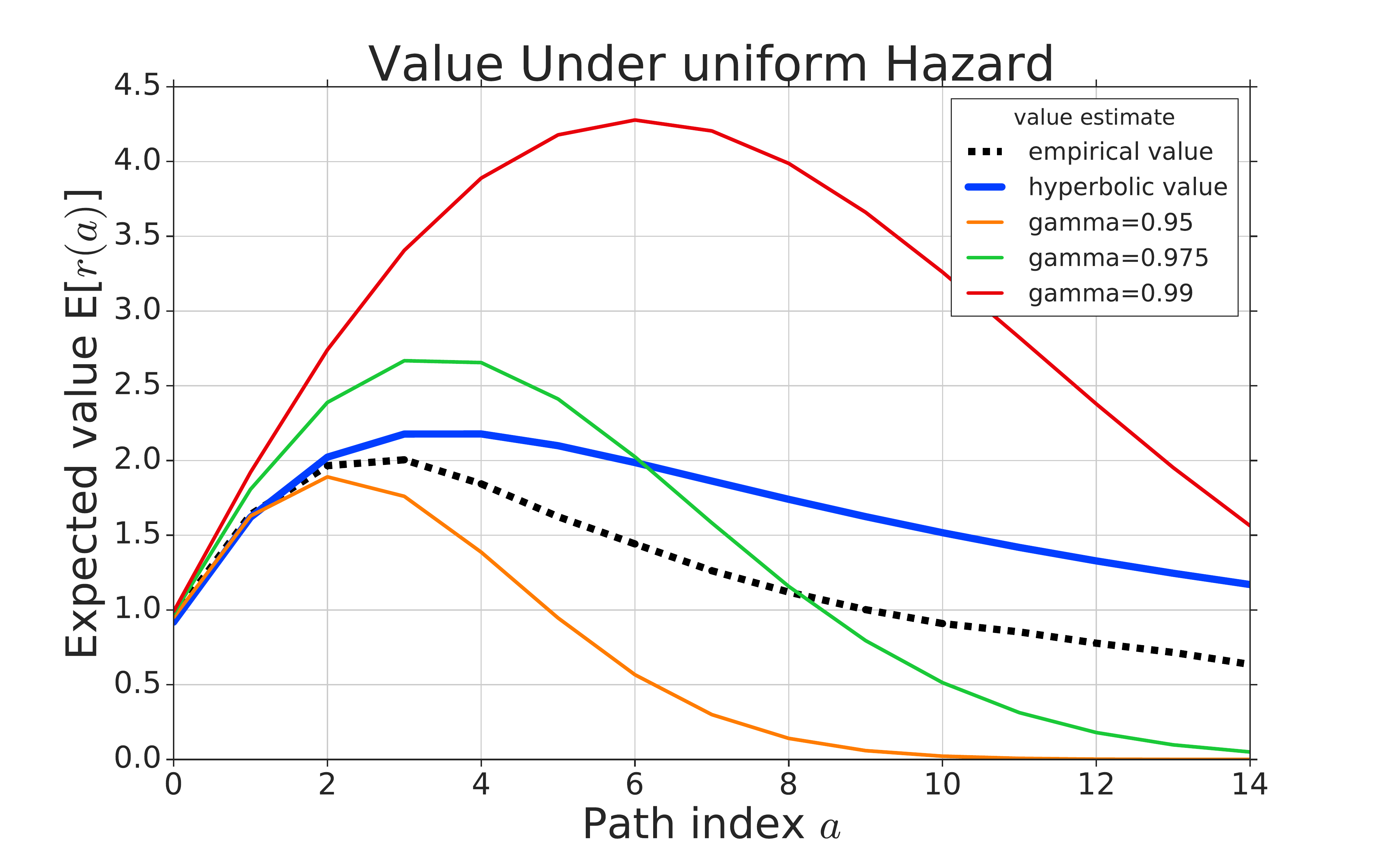}
            \caption{If the true hazard rate is now drawn according to a \emph{uniform} distribution (with the same mean as before) the original hyperbolic discount matches the functional form better than exponential discounting. Numerical results in Table \ref{tbl: uniform_MSE}.}
            \label{fig: uniform_prior}
        }
        
        % Table.
        \capbtabbox{%
            % \small
            \def\arraystretch{1.35}
                \begin{tabular}{ c c } 
                    \textbf{Discount function} & \textbf{MSE} \\ 
                    \hline 
                    hyperbolic value& \textbf{0.235} \\
                    $\gamma=0.975$& 0.266 \\
                    $\gamma=0.95$& 0.470 \\
                    $\gamma=0.99$& 4.029 \\
                    & \\
                    & \\
                    & \\
                \end{tabular}
            }{%
                \caption{The average mean squared error (MSE) over each of the paths in Figure \ref{fig: uniform_prior} when the underlying hazard is drawn according to a \emph{uniform} distribution.  We find that hyperbolic discounting results is more robust to hazards drawn from a uniform distribution than exponential discounting.}
                \label{tbl: uniform_MSE}
            }
    \end{floatrow}
\end{figure}

Through these two validating experiments, we demonstrate the robustness of estimating hyperbolic discounted Q-values in the case when the environment presents dynamic levels of risk and the agent faces non-trivial decisions.
Hyperbolic discounting is preferable to exponential discounting even when the agent's prior does not precisely match the true environment hazard rate distribution, by coefficient (Figure \ref{fig: diff_hyperbolic_coefficients}) or by functional form (Figure \ref{fig: uniform_prior}).

\section{Atari 2600 Experiments}
With our approach validated in Pathworld, we now move to the high-dimensional environment of Atari 2600, specifically, ALE.
We use the Rainbow variant from Dopamine \citep{DBLP:journals/corr/abs-1812-06110} which implements three of the six considered improvements from the original paper:  distributional RL, predicting n-step returns and prioritized replay buffers. 

The agent (Figure \ref{fig: model_architecture}) maintains a shared representation $h(s)$ of state, but computes $Q$-value logits for each of the $N$ $\gamma_i$ via $Q_{\pi}^{(i)}(s,a) = f(W_i h(s) + b_i)$ where $f(\cdot)$ is a ReLU-nonlinearity \citep{nair2010rectified} and $W_i$ and $b_i$ are the learnable parameters of the affine transformation for that head.

\begin{figure}[!h]
    \centering
    \includegraphics[width=0.55\columnwidth]{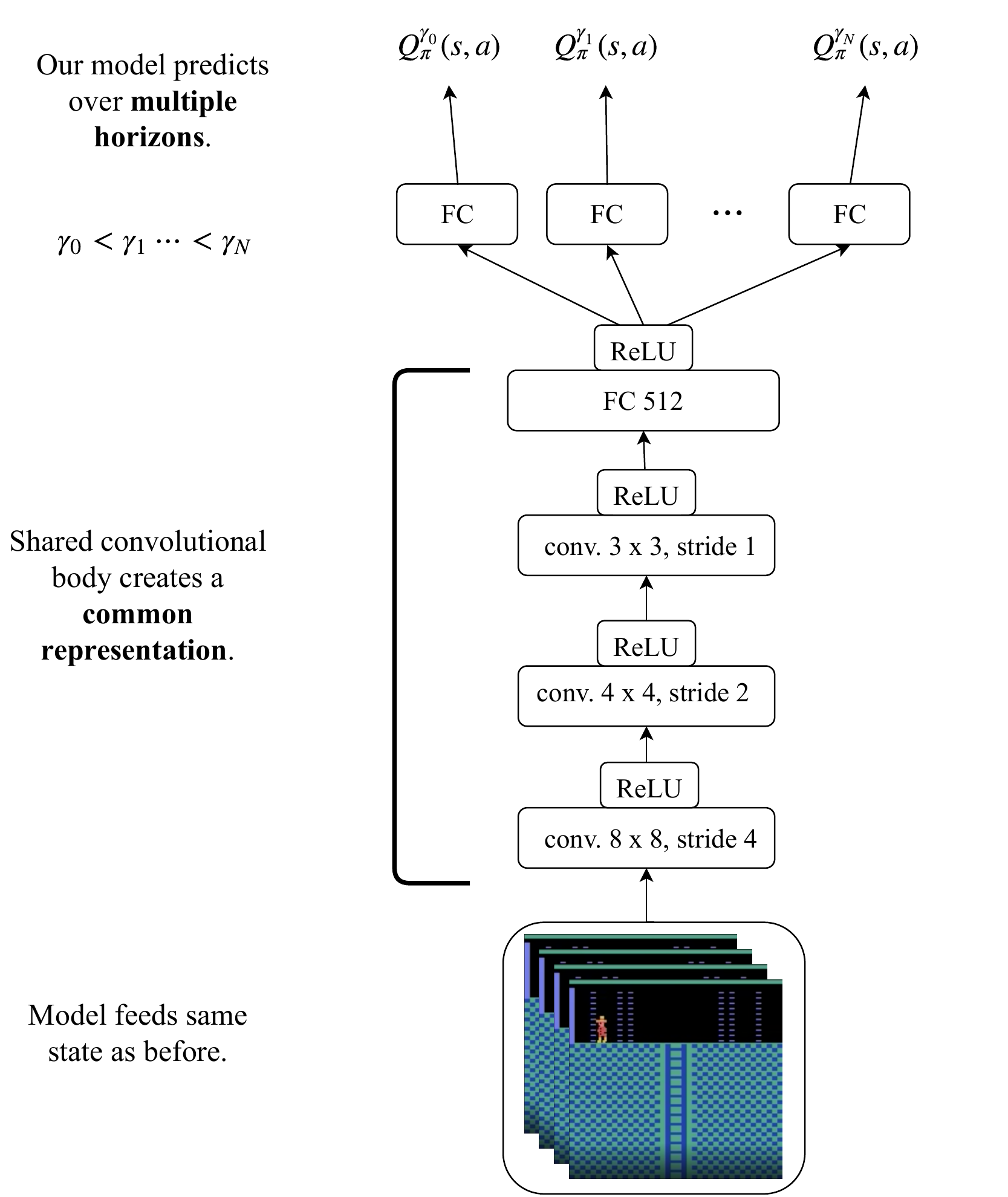}
    \caption{\emph{Multi-horizon} model predicts $Q$-values for $n_{\gamma}$ separate discount functions thereby modeling different effective horizons.  Each $Q$-value is a lightweight computation, an affine transformation off a shared representation.  By modeling over multiple time-horizons, we now have the option to construct \emph{policies} that act according to a particular value or a weighted combination.}
    \label{fig: model_architecture}
\end{figure}

We provide details on the hyperparameters in Appendix \ref{appendix:  hyperparameters}.
We consider the performance of the hyperbolic agent built on Rainbow (referred to as Hyper-Rainbow) on a random subset of Atari 2600 games in Figure \ref{fig: hyperrainbow_vs_multirainbow}.

\begin{figure}[!h]
    \centering
    \includegraphics[width=0.7\columnwidth]{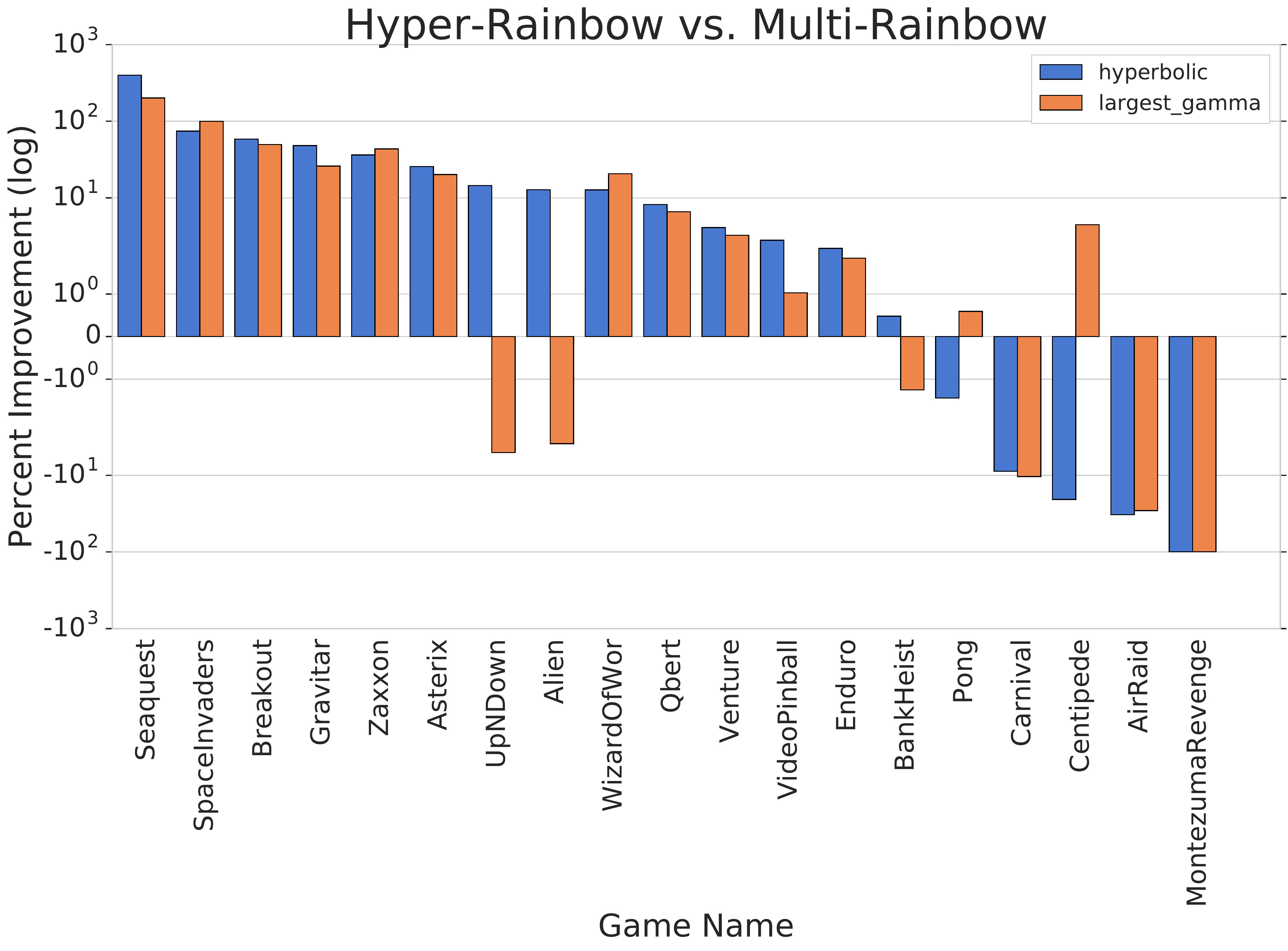}
    \caption{We compare the Hyper-Rainbow (in blue) agent versus the Multi-Rainbow (orange) agent on a random subset of 19 games from ALE  (3 seeds each).  For each game, the percentage performance improvement for each algorithm against Rainbow is recorded.  There is no significant difference whether the agent acts according to hyperbolically-discounted (Hyper-Rainbow) or exponentially-discounted (Multi-Rainbow) Q-values suggesting the performance improvement in ALE emerges from the multi-horizon auxiliary task.}
    \label{fig: hyperrainbow_vs_multirainbow}
\end{figure}

We find that the Hyper-Rainbow agent (blue) performs very well, often improving over the strong-baseline Rainbow agent.
On this subset of 19 games, we find that it improves upon 14 games and in some cases, by large margins.
However, in Section \ref{sec: multi_horizon} we seek a more complete understanding of the \emph{underlying driver} of this improvement in ALE through an ablation study.

\section{Multi-Horizon Auxiliary Task Results}\label{sec: multi_horizon}
To dissect the ALE improvements, recognize that Hyper-Rainbow changes two properties from the base Rainbow agent:

\begin{enumerate}
    \item \textbf{Behavior policy.} The agent acts according to hyperbolic Q-values computed by our approximation described in Section \ref{sec: approx_hyperbolic_qvals}
    \item \textbf{Learn over multiple horizons.} The agent simultaneously learns Q-values over many $\gamma$ rather than a Q-value for a single $\gamma$
\end{enumerate}

The second modification can be regarded as introducing an \emph{auxiliary task} \citep{jaderberg2016reinforcement}.
Therefore, to attribute the performance of each properly we  construct a Rainbow agent augmented with the multi-horizon auxiliary task (referred to as Multi-Rainbow and shown in orange) but have it still act according to the original policy.
 That is, Multi-Rainbow acts to maximize expected rewards discounted by a fixed $\gamma_{action}$ but now learns over multiple horizons as shown in Figure \ref{fig: model_architecture}.

We find that the Multi-Rainbow agent performs nearly as well on these games, suggesting the effectiveness of this as a stand-alone auxiliary task.  
This is not entirely unexpected given the rather special-case of hazard exhibited in ALE through sticky-actions \citep{machado2018revisiting}.
% Revisiting the motivation of hyperbolic discounts, we expect these to confer an advantage when the agent's hazard prior is well-aligned with the true environment hazard, specifically, not a delta-function.
% In ALE, the agent always begins from the same position and the environment repeats relatively predictably, where stochasticity is injected through sticky actions and epsilon-greedy action of the policy.

% Through our development of the Hyperbolic agent we found that modeling multiple time-horizons is an effective \emph{auxiliary task} \citep{jaderberg2016reinforcement}.
We examine further and investigate the performance of this auxiliary task across the full Arcade Learning Environment \citep{bellemare2017distributional} using the recommended evaluation by \citep{machado2018revisiting}.
Doing so we find empirical benefits of the multi-horizon auxiliary task on the Rainbow agent as shown in Figure \ref{fig: aux_rainbow}.

\begin{figure}[!h]
    \centering
    \includegraphics[width=0.5\columnwidth]{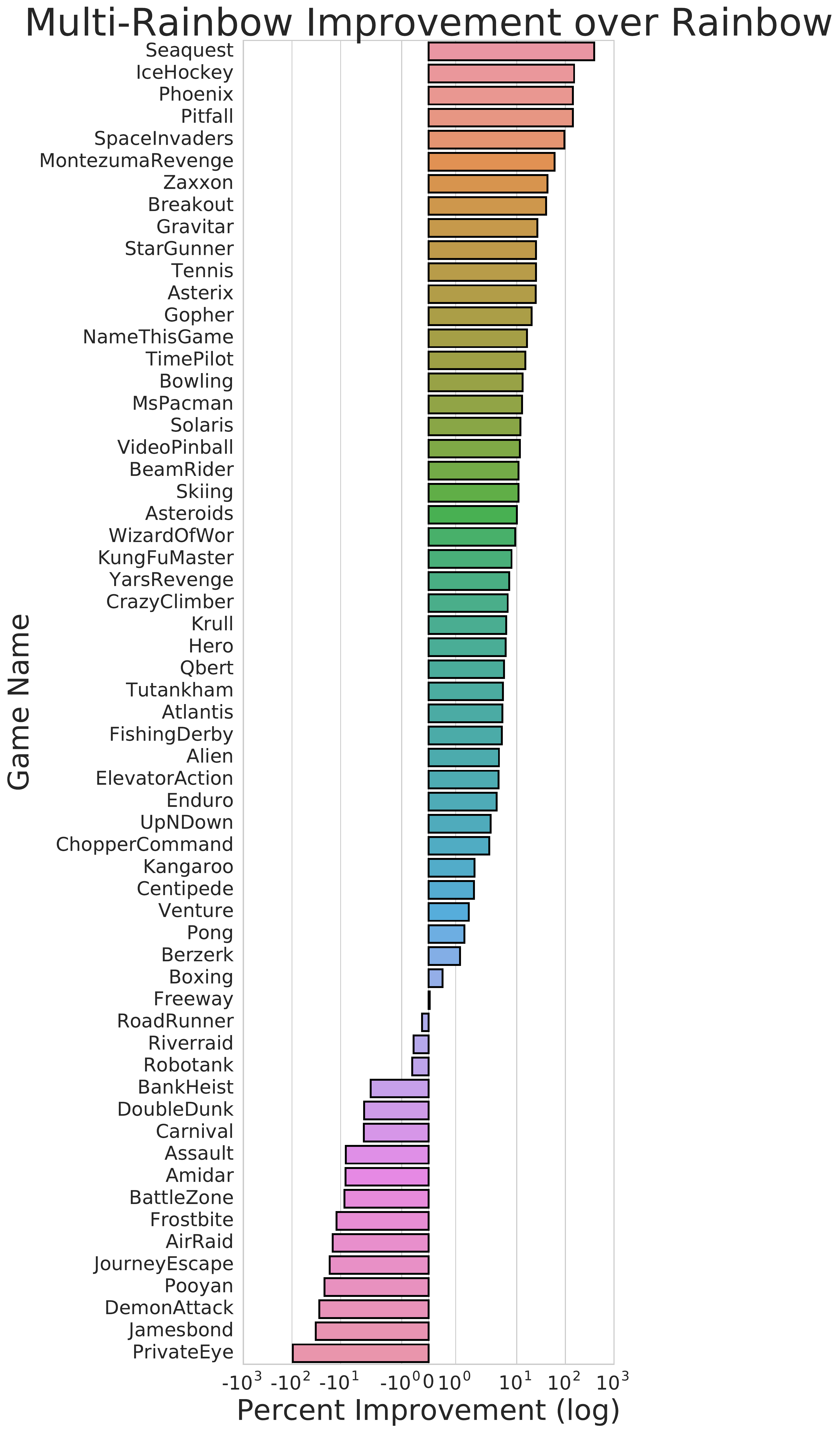}
    \caption{Performance improvement over Rainbow using the multi-horizon auxiliary task in Atari Learning Environment (3 seeds each).}
    \label{fig: aux_rainbow}
\end{figure}

% \subsection{Enabling farther effective horizons}
% In many settings we compare the undiscounted performance of our algorithms, that is $\gamma=1$.  Therefore, on this metric, it may prove beneficial to learn on the largest gamma possible.  
% However, Bellman updates \citep{bertsekas1995neuro} are known to be less-stable as $\gamma \rightarrow 1$.  
% Therefore during training a $\gamma < 1$ is used. 
% This discounting scheme and value $\gamma$ induces an effective time-horizon on the MDP of $h=1/(1-\gamma)$.  
% Note that $\lim_{\gamma\to1} \gamma^h= e^{-1}$ which is the distance at which a future reward is discounted by $1/e$. Rewards received at times far beyond this horizon $t >> h$ are neglected in the estimate of the current value.
% This means that our training procedure may not produce the desired results for our deployed system at $\gamma=1$.
% These issues of limited effective horizon arise in sophisticated, temporally extended games such as Dota \citep{OpenAI_dota}.
% Annealing the discount factor $\gamma$ from a low value to a high value is used as a training modification \citep{OpenAI_dota}.
% To remedy this issue, we find that by training the multi-horizon auxiliary tasks that we are able to reliably learn and act according to longer horizons.  

% \liam{record baseline model with $\gamma=0.999$ or higher}

% \liam{for Atari, the results are mixed.  performing with $\gamma=0.99$ usually does quite well.  is acting according to larger gammas just for the sake of it sufficiently interesting?}

\subsection{Analysis and Ablation Studies}
To understand the interplay of the multi-horizon auxiliary task with other improvements in deep RL, we test a random subset of 10 Atari 2600 games against improvements in Rainbow \citep{hessel2018rainbow}.  
On this set of games we measure a consistent improvement with multi-horizon C51 (Multi-C51) in 9 out of the 10 games over the base C51 agent \citep{bellemare2017distributional} in Figure \ref{fig: ablation_study}.  

% Ablation study games:
% 'Asterix', 'Breakout', 'Enduro', 'Pong', 'Seaquest',
% 'UpNDown', 'Venture', 'Gravitar', 'Qbert', 'Zaxxon']

% but we delineate here the degree to which it improves over a cross-product n-step returns and prioritized replay buffer.

% \begin{center}
%     \small
%     \bgroup
%         \def\arraystretch{1.5}
%         \begin{tabular}{ |c|c|c| } 
%          \hline
%           & \textbf{1-step} & \textbf{n-step} \\ 
%          \hline 
%          \textbf{uniform}        & Multi-C51                  & Multi-C51 + n-step   \\ 
%          \textbf{prioritized}    & Multi-C51 + priority       & Multi-Rainbow    \\ 
%          \hline
%         \end{tabular}
%         \captionof{table}{Cross-product of (n-step $\times$ prioritized replay buffer) with the multi-horizon auxiliary task.}
%     \egroup
% \end{center}

\begin{figure}[h]
    \centering
    \subfigure[Multi-C51]{\includegraphics[width=0.35\columnwidth]{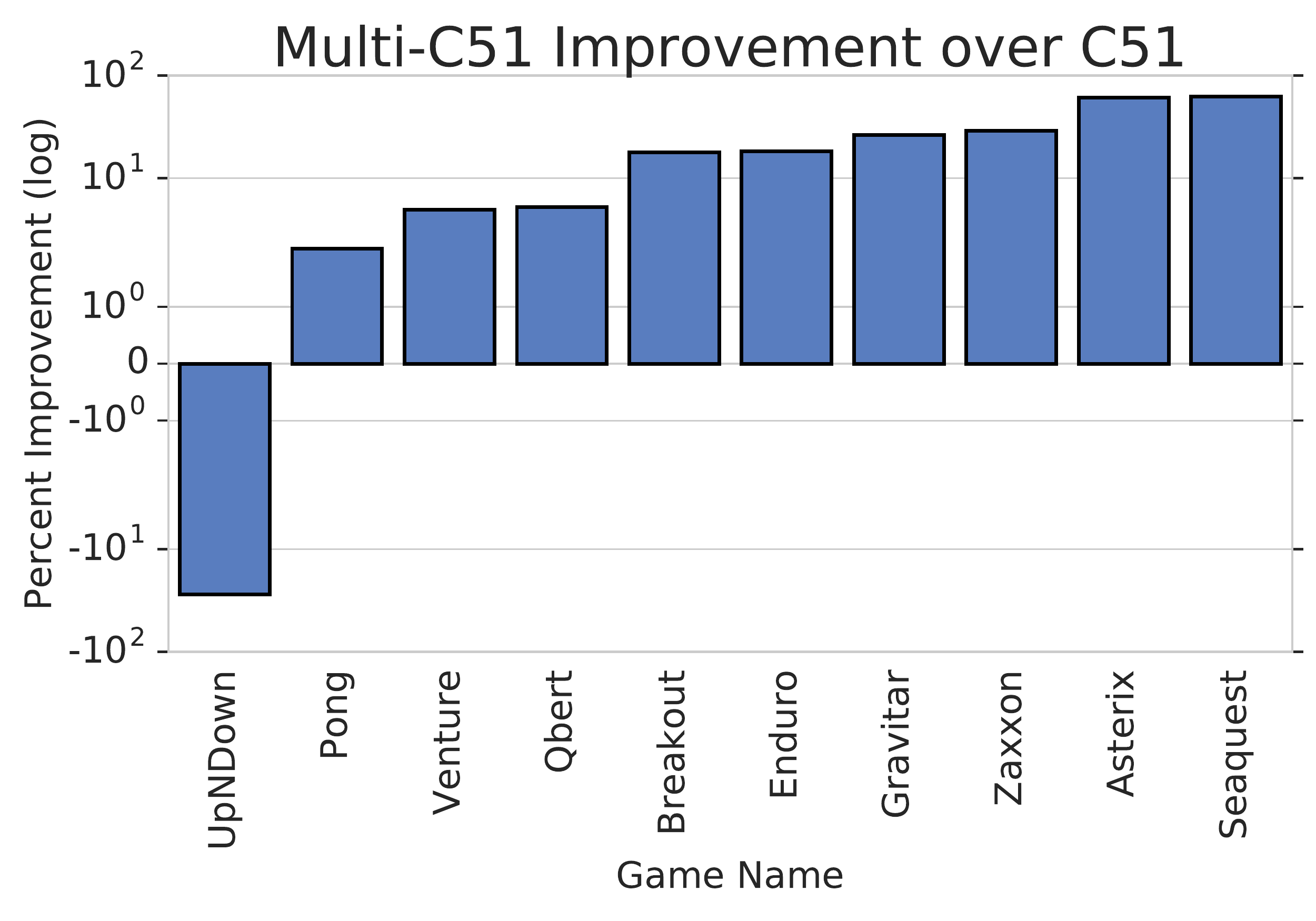}}
    \subfigure[Multi-C51 + n-step]{\includegraphics[width=0.35\columnwidth]{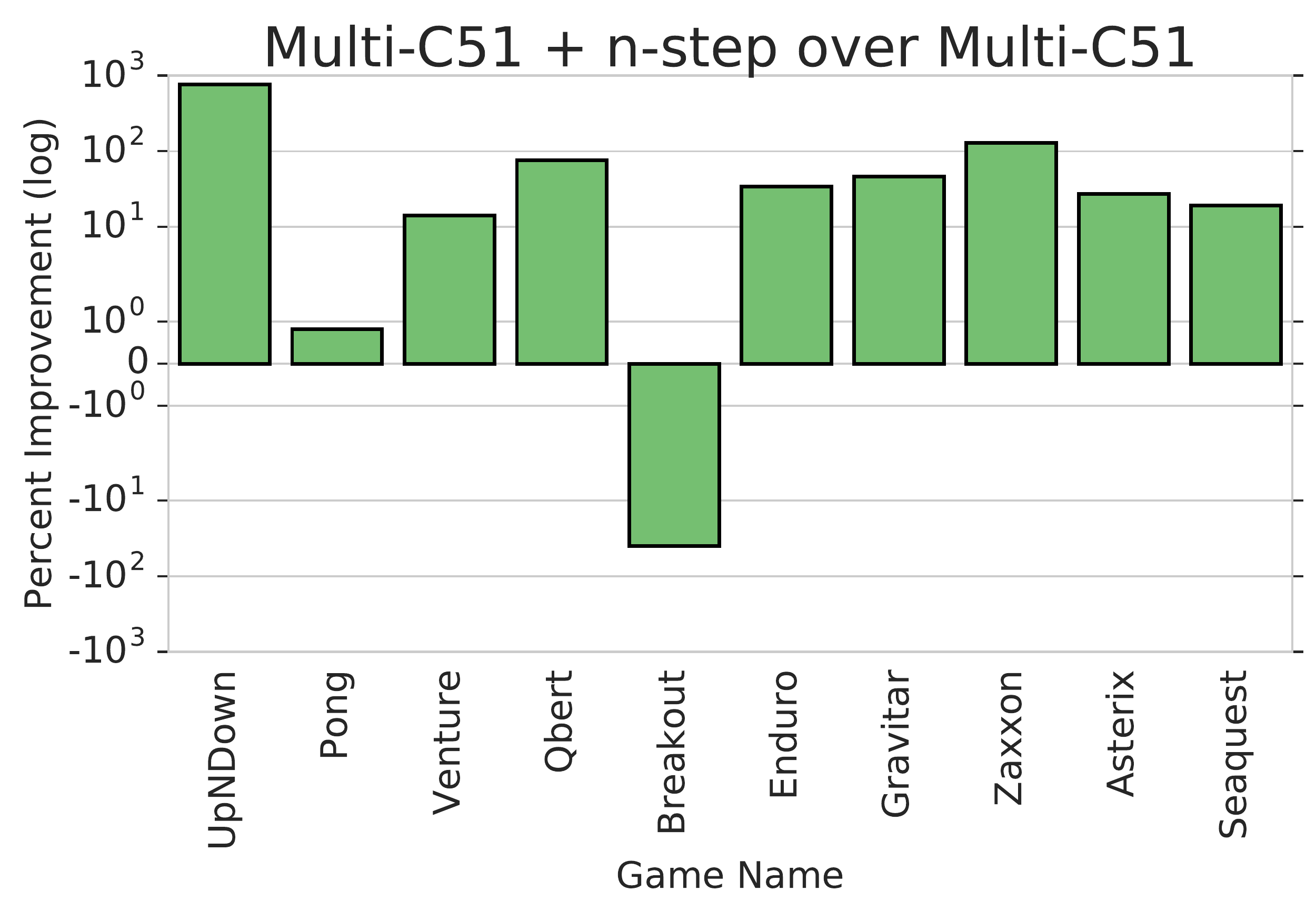}}
    \vskip\baselineskip
    \subfigure[Multi-C51 + priority]{\includegraphics[width=0.35\columnwidth]{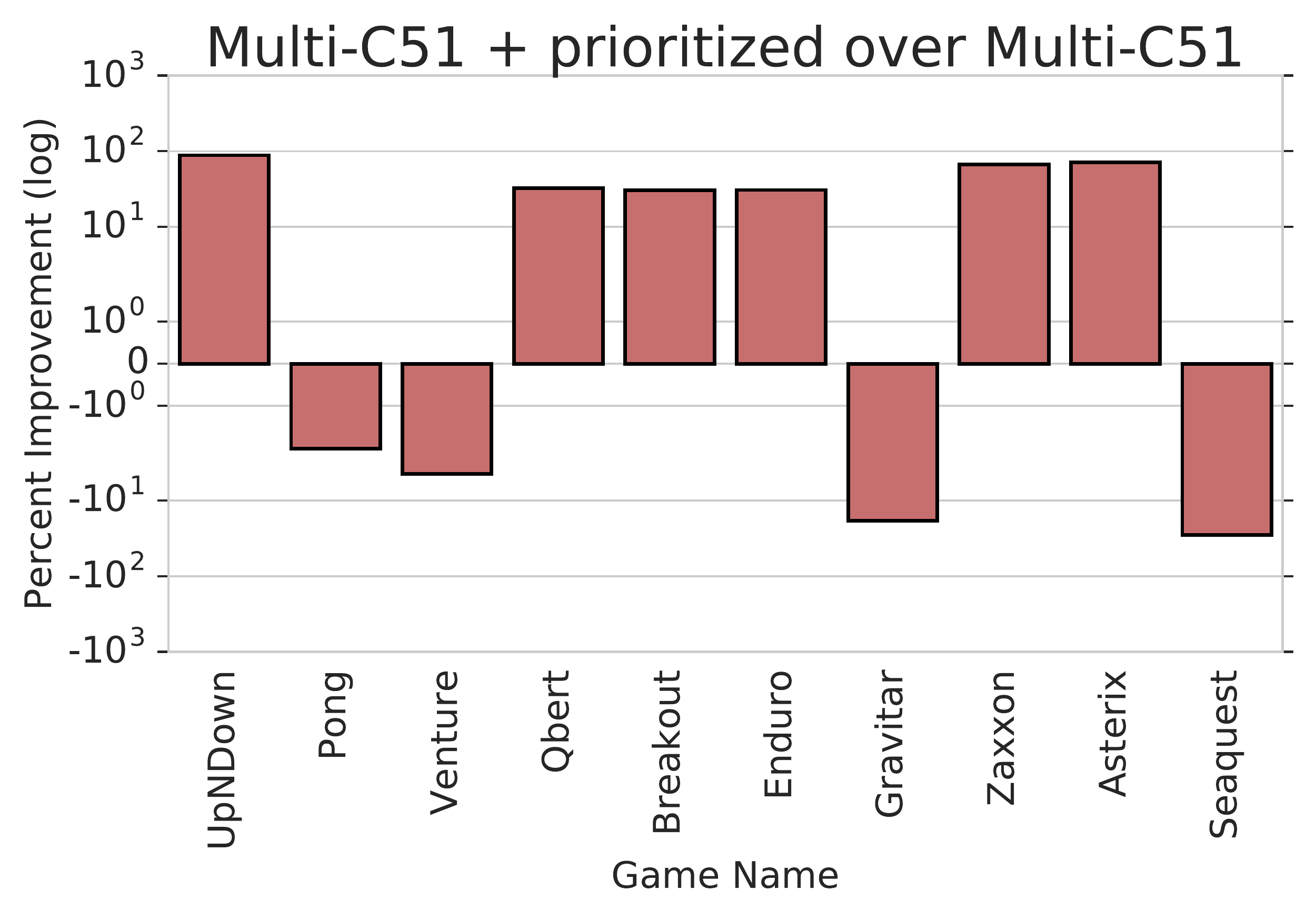}}
    \subfigure[Multi-Rainbow (=Multi-C51 + n-step + priority)]{\includegraphics[width=0.35\columnwidth]{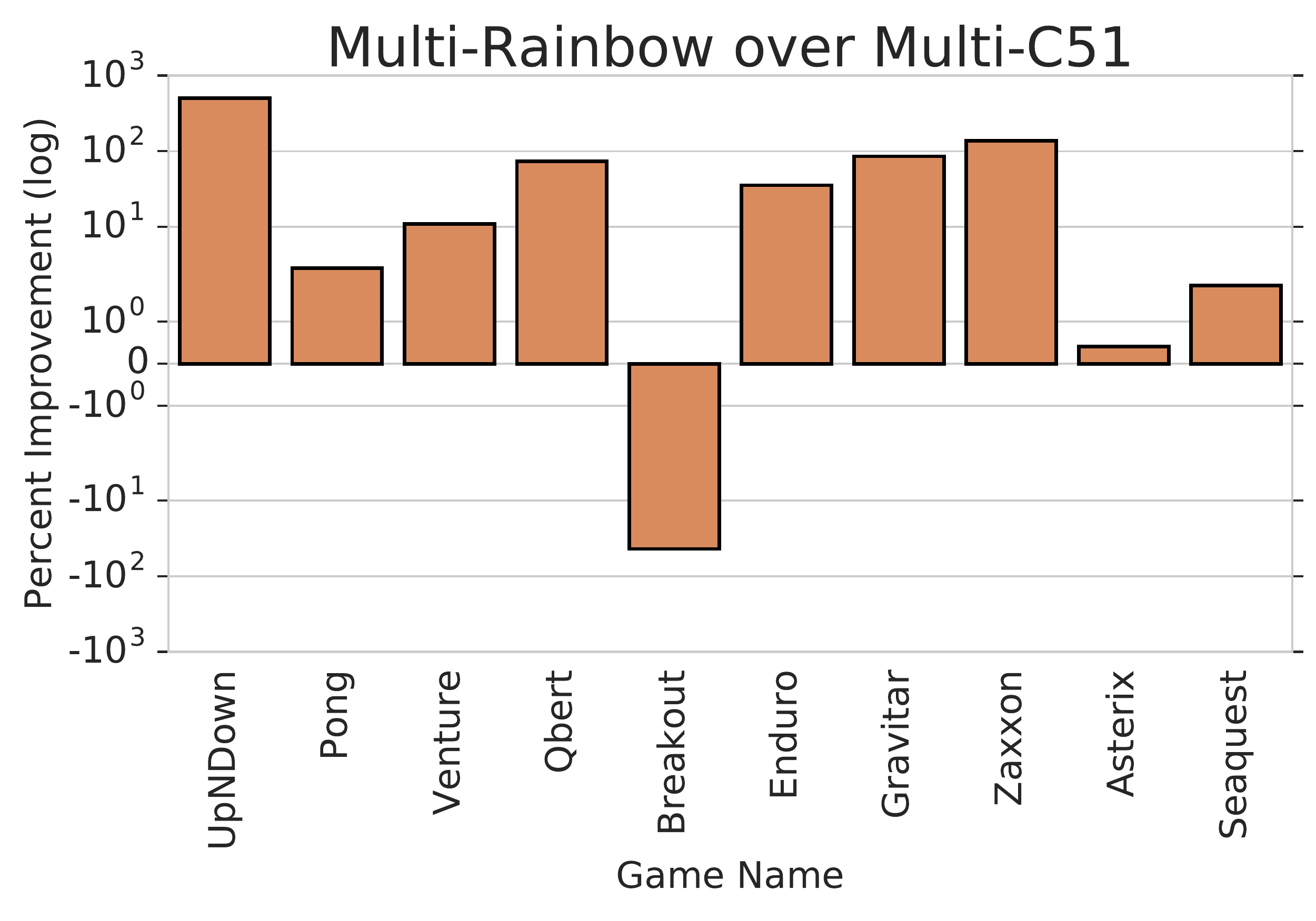}}
    \caption{Measuring the Rainbow improvements on top of the Multi-C51 baseline on a subset of 10 games in the Arcade Learning Environment (3 seeds each). On this subset, we find that the multi-horizon auxiliary task interfaces well with n-step methods (top right) but poorly with a prioritized replay buffer (bottom left).}
    \label{fig: ablation_study}
\end{figure}

Figure \ref{fig: ablation_study} indicates that the current implementation of Multi-Rainbow does not generally build successfully on the prioritized replay buffer.  On the subset of ten games considered, we find that four out of ten games (Pong, Venture, Gravitar and Zaxxon) are negatively impacted despite \citep{hessel2018rainbow} finding it to be of considerable benefit and specifically beneficial in three out of these four games (Venture was not considered).  The current prioritization scheme simply averaged the temporal-difference errors over all $Q$-values to establish priority.  
Alternative prioritization schemes are offering encouraging preliminary results (Appendix \ref{appendix: alternative_priority}).

\section{Discussion}
This work builds on a body of work that questions one of the basic premises of RL:  one should maximize the \emph{exponentially discounted} returns via a \emph{single} discount factor.
By learning over multiple horizons simultaneously, we have broadened the scope of our learning algorithms.
Through this we have shown that we can enable acting according to new discounting schemes and that learning multiple horizons is a powerful stand-alone auxiliary task.  
Our method well-approximates hyperbolic discounting and performs better in hazardous MDP distributions.
This may be viewed as part of an algorithmic toolkit to model alternative discount functions.
% but take no stance specifically justifying hyperbolic discounting. other than that it is arises in biological systems and may thus be useful for the types of risk encountered in our world.

\section{Future Work}
There is growing interest in the time-preferences of RL agents.  
Through this work we have considered models of a constant, albeit uncertain, hazard rate $\lambda$.
This moves beyond the canonical RL approach of fixing a single $\gamma$ which implicitly holds no uncertainty on the value of $\lambda$ but this still does not fully capture all aspects of risk since the hazard rate may be a function of time.
Further, hazard may not be an intrinsic property of the environment but a joint property of both the \emph{policy} and the environment.
If an agent purses a policy leading to dangerous state distributions then it will naturally be subject to higher hazards and vice-versa.
We would therefore expect an interplay between time-preferences and policy.
% The agent should reason about time-preference differently between safe and hazardous states.
% If the states are low-hazard, areas of state space the agent can navigate reliably, then it will estimate a correspondingly low $\lambda$ and thus a long horizon.
% However, when the agent is exploring new dangerous states, its estimate of $\lambda$ and corresponding horizon should adjust.
% If the agent is choosing dangerous states 
This is not simple to deal with but recent work proposing state-action dependent discounting \citep{pitis2019rethinking} may provide a formalism for more general time-preference schemes.

\section*{Acknowledgements}
This research and its general framing drew upon the talents of many researchers at Google Brain, DeepMind and Mila.  
In particular, we'd like thank Ryan Sepassi for framing of the paper, Utku Evci for last minute Matplotlib help, Audrey Durand, Margaret Li, Adrien Ali Taïga, Ofir Nachum, Doina Precup, Jacob Buckman,  Marcin Moczulski, Nicolas Le Roux, Ben Eysenbach, Sherjil Ozair,  Anirudh Goyal, Ryan Lowe, Robert Dadashi, Chelsea Finn, Sergey Levine, Graham Taylor and Irwan Bello for general discussions and revisions.

\clearpage
\bibliography{references}
\bibliographystyle{iclr2019_conference}

\clearpage
\appendix
\section{Equivalence of Hyperbolic Discounting and Exponential Hazard}\label{appendix:  hyperbolic_equiv_hazard}
Following Section \ref{sec: hazard_implies_discount} we also show a similar equivalence between hyperbolic discounting and the specific hazard distribution
$p_k(\lambda) = \frac{1}{k} \text{exp} (-\lambda / k)$, where again, $\lambda \in [0, \infty)$
\begin{align*}
Q_\pi^{\delta(0), \Gamma_k} (s,a) &= \mathbb{E}_{ \pi, P_0} \left[ \sum_{t=0}^\infty \Gamma_k(t)  R(s_t, a_t) | s_0 = s, a_0 = a \right] \\
&= \mathbb{E}_{ \pi, P_0} \left[ \sum_{t=0}^\infty \left ( \int_{\lambda=0}^\infty  p_k(\lambda) e^{-\lambda t} d \lambda \right )  R(s_t, a_t) | s_0 = s, a_0 = a \right] \\
&= \int_{\lambda=0}^\infty  p_k(\lambda) \mathbb{E}_{ \pi, P_0} \left[ \sum_{t=0}^\infty  e^{-\lambda t}R(s_t, a_t) | s_0 = s, a_0 = a \right]  d \lambda \\
&= \mathbb{E}_{\lambda \sim p_k(\cdot)} \mathbb{E}_{ \pi, P_0} \left[ \sum_{t=0}^\infty  e^{-\lambda t}R(s_t, a_t) | s_0 = s, a_0 = a \right]  \\
&= \mathbb{E}_{\lambda \sim p_k(\cdot)} \mathbb{E}_{ \pi, P_\lambda} \left[ \sum_{t=0}^\infty R(s_t, a_t) | s_0 = s, a_0 = a \right]  \\
&= Q_\pi^{p_k, 1} (s,a) \\ 
\end{align*}
Where the first step uses Equation \ref{eqn: hyp_discount}. This equivalence implies that discount factors can be used to learn policies that are robust to hazards. 

\section{Alternative Discount Functions}\label{appendix: alternative_discount}
We expand upon three special cases to see how functions $f(\gamma, t) = w(\gamma)\gamma^t$ may be related to different discount functions $d(t)$.

\textbf{Three cases:}
\vspace{-0.2cm}
\begin{enumerate}
    \item \textbf{Delta hazard prior}:  $p(\lambda) = \delta(\lambda - k)$
    \item \textbf{Exponential hazard prior}: $p(\lambda) = \frac{1}{k}e^{-\lambda / k}$
    \item \textbf{Uniform hazard prior}: $p(\lambda) = \frac{1}{k}$ for $\lambda \in [0, k]$
\end{enumerate}

For the three cases we begin with the Laplace transform on the prior $p(\lambda) = \int_{\lambda=0}^\infty  p(\lambda) e^{-\lambda t} d\lambda$ and then chnage the variables according to the relation between $\gamma = e^{-\lambda}$, Equation \ref{eqn: gamma_to_lambda}.

%%%%%%%%%%%%
% Example 1.
%%%%%%%%%%%%
\subsection{Delta Hazard Prior}
A delta prior $p(\lambda) = \delta(\lambda - k)$ on the hazard rate is consistent with exponential discounting.
\begin{align*}
    \int_{\lambda=0}^\infty  p(\lambda) e^{-\lambda t} d\lambda &= \int_{\lambda=0}^\infty  \delta(\lambda - k) e^{-\lambda t} d\lambda\\
         &= e^{-kt}
\end{align*}
where $\delta(\lambda - k)$ is a Dirac delta function defined over variable $\lambda$ with value $k$. The change of variable $\gamma=e^{-\lambda}$ (equivalently $\lambda=-\ln\gamma$) yields differentials $d\lambda = -\frac{1}{\gamma} d\gamma$ and the limits $\lambda=0 \rightarrow  \gamma=1$ and $\lambda = \infty \rightarrow \gamma = 0$. 
Additionally, the hazard rate value $\lambda=k$ is equivalent to the $\gamma=e^{-k}$.
\begin{align*}
    d(t) &= \int_{\lambda=0}^\infty  p(\lambda) e^{-\lambda t} d\lambda \\
    &= \int_{\gamma=1}^0  \delta(-\ln \gamma - k) \gamma ^t \left(-\frac{1}{\gamma} d\gamma \right)\\
    &= \int_{\gamma=0}^1  \delta(-\ln \gamma - k) \gamma ^{t-1} d\gamma \\
    &= e^{-kt} \\
    &= \gamma_k^t
\end{align*}

where we define a $\gamma_k = e^{-k}$ to make the connection to standard RL discounting explicit.
% However, this process holds for more complicated cases as we will see.
Additionally and reiterating, the use of a single discount factor, in this case $\gamma_k$, is equivalent to the prior that a \emph{single} hazard exists in the environment.

%%%%%%%%%%%%
% Example 2.
%%%%%%%%%%%%
\subsection{Exponential Hazard Prior}
Again, the change of variable $\gamma=e^{-\lambda}$ yields differentials $d\lambda = -\frac{1}{\gamma} d\gamma$ and the limits $\lambda=0 \rightarrow  \gamma=1$ and $\lambda = \infty \rightarrow \gamma = 0$.
\begin{align*}
    \int_{\lambda=0}^\infty  p(\lambda) e^{-\lambda t} d\lambda &= \int_{\gamma=1}^0 p(-\text{ln}\gamma)\gamma ^t \left(-\frac{1}{\gamma} d\gamma \right)\\
    &= \int_{\gamma=0}^1  p(-\text{ln}\gamma)\gamma ^{t-1} d\gamma
\end{align*}
where $p(\cdot)$ is the prior.  With the exponential prior $p(\lambda) = \frac{1}{k} \text{exp}(-\lambda / k)$ and by substituting $\lambda=-\text{ln}\gamma$ we verify Equation \ref{eqn: hyperbolic}

\begin{align*}\label{eqn: hyperbolic_through_gamma}
    \int_0^1 \frac{1}{k} \text{exp}(\ln \gamma / k)  \gamma^{t-1} d\gamma 
        &=  \frac{1}{k} \int_0^1  \text{exp}(\text{ln} \gamma^{1/k}) \gamma^{t-1} d\gamma\\
        &=  \frac{1}{k} \int_0^1  \gamma ^{1/k + t-1} d\gamma\\
        &= \frac{1}{k} \frac{1}{\frac{1}{k} + t} \gamma^{1/k + t} \biggr\rvert_{\gamma=0}^1 \\
        &= \frac{1}{1 + kt}  
\end{align*}

%%%%%%%%%%%%
% Example 3.
%%%%%%%%%%%%
\subsection{Uniform Hazard Prior}
Finally if we hold a uniform prior over hazard, $\frac{1}{k}$ for  $\lambda \in [0, k]$ then \cite{sozou1998hyperbolic} shows the Laplace transform yields 
\begin{align*}
    d(t) =&  \int_0^\infty p(\lambda) e^{-\lambda t} d\lambda \\
         =& \frac{1}{k} \int_0^k  e^{-\lambda t} d\lambda\\
         =& -\frac{1}{kt} e^{-\lambda t} \biggr\vert_{\lambda=0}^k \\
         =& \frac{1}{kt} \left(1 - e^{-kt} \right)
\end{align*}
Use the same change of variables to relate this to $\gamma$.  The bounds of the integral become $\lambda=0 \rightarrow  \gamma=1$ and $\lambda = k \rightarrow \gamma = e^{-k}$.
\begin{align*}
    d(t) =& -\frac{1}{k} \int_{\gamma=1}^{e^{-k}}  \gamma^{t-1} d\gamma\\
         =& \frac{1}{kt} \gamma^t \biggr\vert_{\gamma=e^{-k}}^1 \\
         =& \frac{1}{kt} \left(1 - e^{-kt} \right)
\end{align*}
which recovers the discounting scheme.  

\section{Determining the $\gamma$ Interval}\label{appendix: gamma_interval}
We provide further detail for which $\gamma$ we choose to model and motivation why.
We choose a $\gamma_\text{max}$ which is the largest $\gamma$ to learn through Bellman updates.  
If we are using $k$ as the hyperbolic coefficient in Equation \ref{eqn: hyp_discount} and we are approximating the integral with $n_\gamma$ our $\gamma_\text{max}$ would be
\begin{equation}
\gamma_{\text{max}}  =  \left(1 - b ^ {n_\gamma} \right)^k 
\label{eqn: gamma_max}
\end{equation}
However, allowing $\gamma_\text{max} \rightarrow 1$ get arbitrarily close to 1 may result in learning instabilities \cite{bertsekas1995neuro}.
Therefore we compute an exponentiation base of $b = \text{exp}(\text{ln}(1 - \gamma_{\text{max}} ^ {1/ k }) / n_{\gamma})$ which bounds our $\gamma_\text{max}$ at a known stable value. 
This induces an approximation error which is described more in Appendix \ref{appendix:  approximation_error}.

\section{Estimating Hyperbolic Coefficients}\label{appendix:  hyperbolic_approach}
As discussed, we can estimate the hyperbolic discount in two different ways.
We illustrate the resulting estimates here and resulting approximations.
We use lower-bound Riemann sums in both cases for simplicity but more sophisticated integral estimates exist.

\begin{figure}[h]
    \subfigure[Our approach.]{\includegraphics[width=0.45\columnwidth]{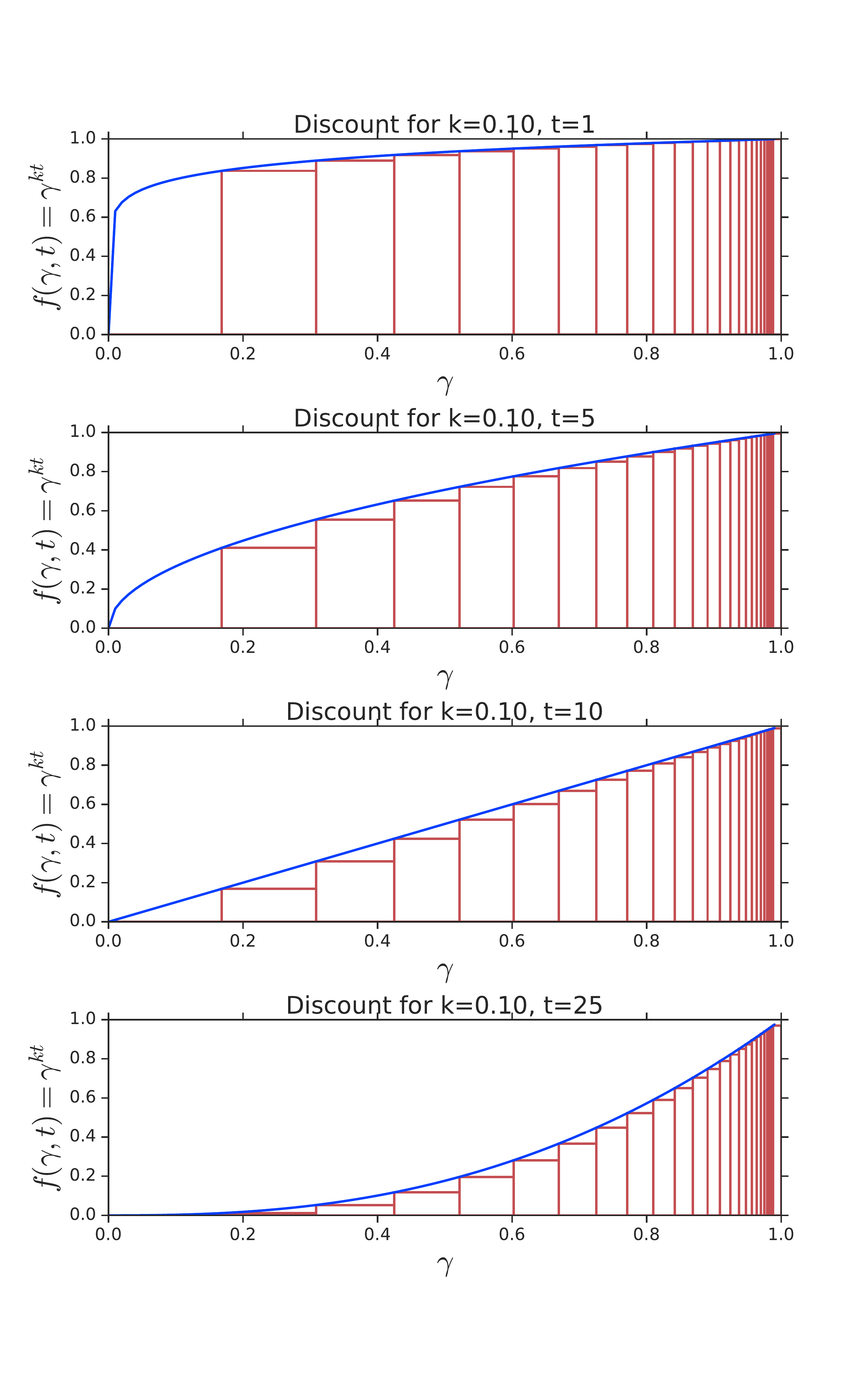}}
    \subfigure[Alternative approach.]{\includegraphics[width=0.45\columnwidth]{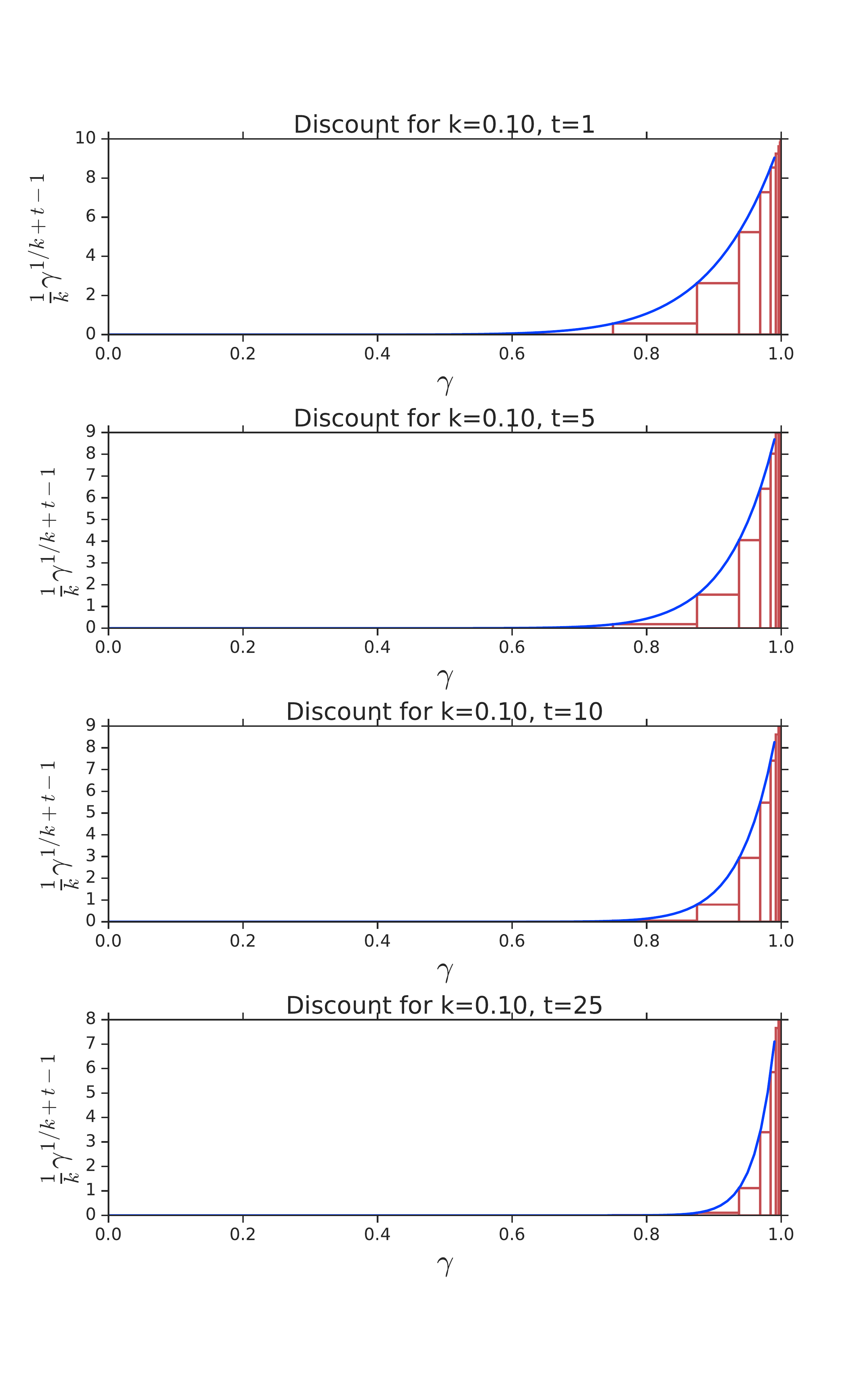}}
    
    \caption{Comparison of hyperbolic coefficient integral estimation between the two approaches.  \\
    (a) We approximate the integral of the function $\gamma^{kt}$ via a lower estimate of rectangles at specific $\gamma$-values.  The sum of these rectangles approximates the hyperbolic discounting scheme $1/(1 + kt)$ for time $t$.  \\
    (b) Alternative form for approximating hyperbolic coefficients which is sharply peaked as $\gamma \rightarrow 1$ which led to larger errors in estimation under our initial techniques.}
    \label{fig: compare_integration_schemes}
\end{figure}

As noted earlier, we considered two different integrals for computed the hyperbolic coefficients.
Under the form derived by the Laplace transform, the integrals are sharply peaked as $\gamma \rightarrow 1$.
The difference in integrals is visually apparent comparing in Figure \ref{fig: compare_integration_schemes}.

\section{Performance of Different Replay Buffer Prioritization Scheme}\label{appendix: alternative_priority}
As found through our ablation study in Figure \ref{fig: ablation_study}, the Multi-Rainbow auxiliary task interacted poorly with the prioritized replay buffer when the TD-errors were averaged evenly across all heads.  As an alternative scheme, we considered prioritizing according to the largest $\gamma$, which is also the $\gamma$ defining the $Q$-values by which the agent acts.

The (preliminary\footnote{These runs have been computed over approximately 100 out of 200 iterations and will be updated for the final version.}) results of this new prioritization scheme is in Figure \ref{fig: aux_rainbow_largest}.

\begin{figure}[!ht]
    \centering
    \includegraphics[width=0.7\columnwidth]{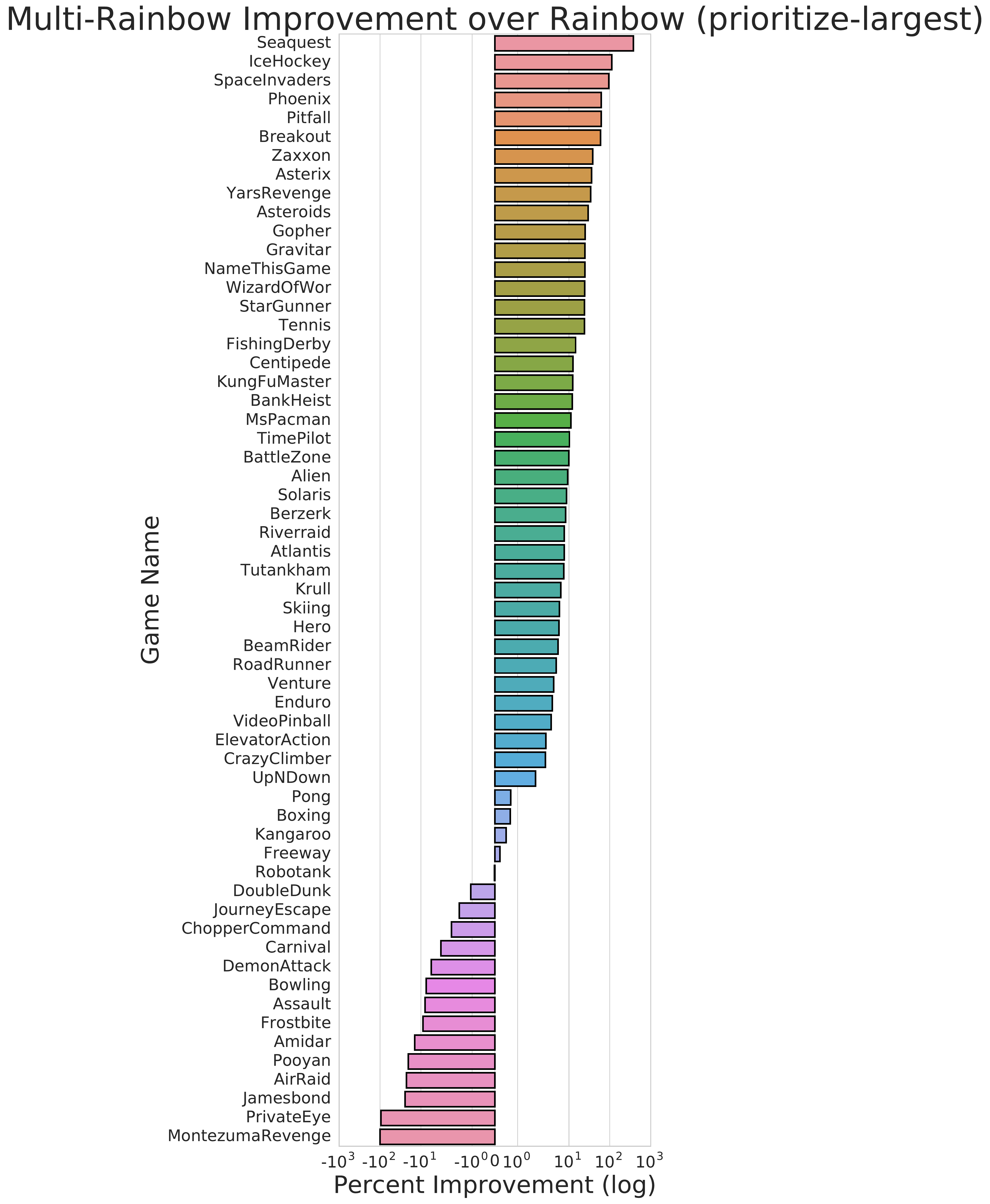}
    \caption{The (preliminary) performance improvement over Rainbow using the multi-horizon auxiliary task in Atari Learning Environment when we instead prioritize according to the TD-errors computed from the largest $\gamma$ (3 seeds each).}
    \label{fig: aux_rainbow_largest}
\end{figure}

To this point, there is evidence that prioritizing according to the TD-errors generated by the largest gamma is a better strategy than averaging.

\section{Approximation Errors}\label{appendix:  approximation_error}
Instead of evaluating the upper bound of Equation \ref{eqn: hyperbolic} at 1 we evaluate at $\gamma_{\text{max}}$ which yields $\gamma_{\text{max}}^{kt} / (1 + kt)$.  Our approximation induces an error in the approximation of the hyperbolic discount.  

\begin{figure}[h]
    \begin{floatrow}
        % Figure.
        \ffigbox[0.6\FBwidth]{%
            \includegraphics[width=\columnwidth]{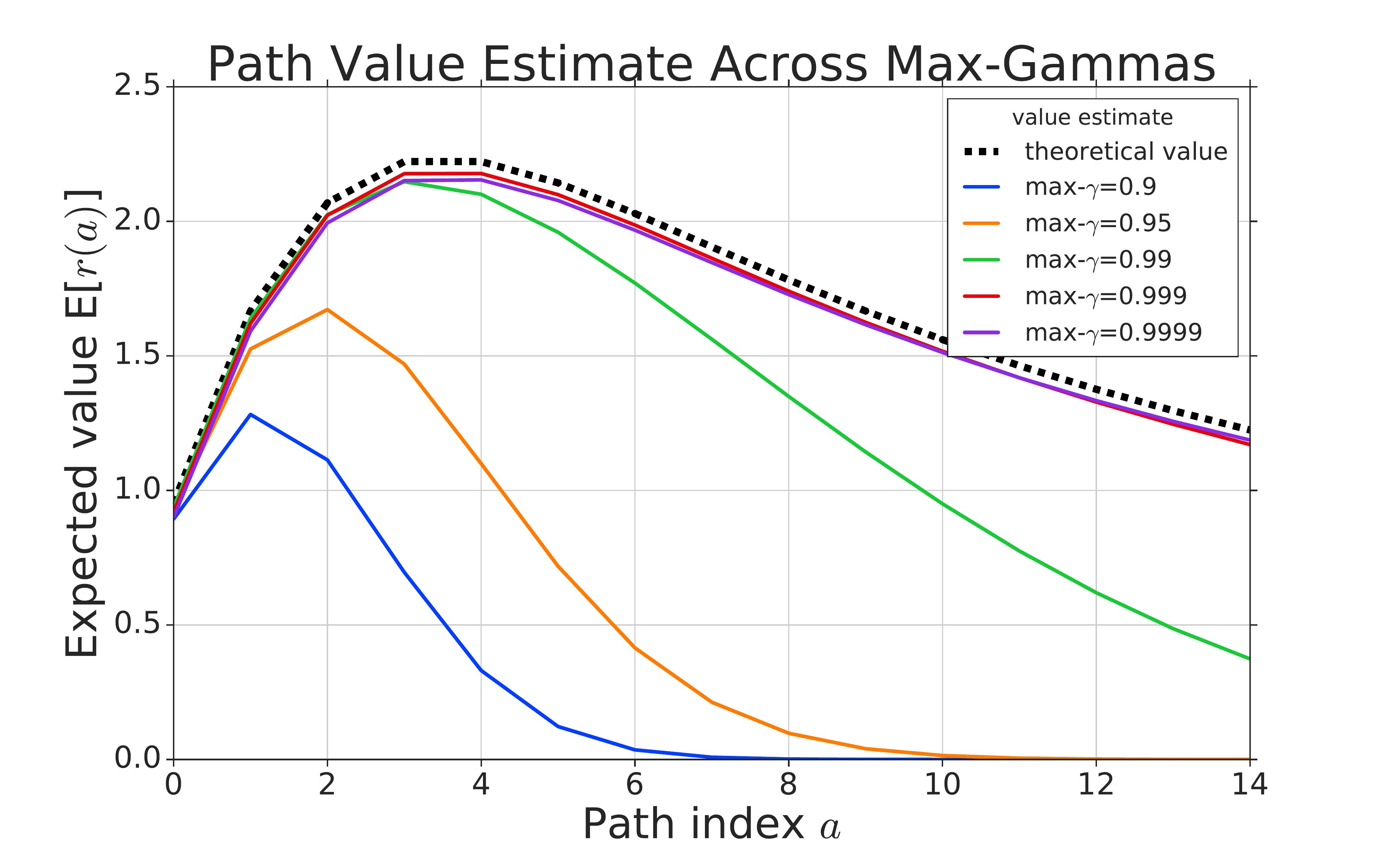}
            \caption{By instead evaluating our integral up to $\gamma_{\text{max}}$ rather than to 1, we induce an approximation error which increases with $t$. Numerical results in Table \ref{tbl: max_gamma_MSE}.}
            \label{fig: max_gamma}
        }
        
        % Table.
        \capbtabbox{%
            % \small
            \def\arraystretch{1.35}
                \begin{tabular}{ c c } 
                    \textbf{Discount function} & \textbf{MSE} \\ 
                    \hline 
                    max-$\gamma$=0.999& \textbf{0.002} \\
                    max-$\gamma$=0.9999& 0.003 \\
                    max-$\gamma$=0.99& 0.233 \\
                    max-$\gamma$=0.95& 1.638 \\
                    max-$\gamma$=0.9& 2.281 \\
                    & \\
                    & \\
                \end{tabular}
            }{%
                \caption{The average mean squared error (MSE) over each of the paths in Figure \ref{fig: max_gamma}.}
                \label{tbl: max_gamma_MSE}
            }
    \end{floatrow}
\end{figure}

This approximation error in the Riemann sum increases as the $\gamma_{\text{max}}$ decreases as evidenced by Figure \ref{fig: max_gamma}.
When the maximum value of $\gamma_\text{max} \rightarrow 1$ then the approximation becomes more accurate as supported in Table \ref{tbl: max_gamma_MSE} up to small random errors.

\clearpage
\onecolumn

\section{Hyperparameters}\label{appendix:  hyperparameters}
For all our experiments in DQN \cite{mnih2015human}, C51 \cite{bellemare2017distributional} and Rainbow \cite{hessel2018rainbow}, we benchmark against the baselines set by \cite{DBLP:journals/corr/abs-1812-06110} and we use the default hyperparameters for each of the respective algorithms.  That is, our Multi-agent uses the same optimization, learning rates, and hyperparameters as it's base class.

\begin{table}[h!]
    \small
    \centering
    \begin{tabular}{c c}
    Hyperparameter & Value \\
    \hline \hline
    \texttt{Runner.sticky\_actions} & \texttt{Sticky actions prob 0.25} \\
    \texttt{Runner.num\_iterations} & \texttt{200} \\
    \texttt{Runner.training\_steps} & \texttt{250000}  \\
    \texttt{Runner.evaluation\_steps} & \texttt{125000} \\
    \texttt{Runner.max\_steps\_per\_episode} & \texttt{27000} \\
    & \\
    \texttt{WrappedPrioritizedReplayBuffer.replay\_capacity} & \texttt{1000000} \\
    \texttt{WrappedPrioritizedReplayBuffer.batch\_size} & \texttt{32} \\
    & \\
    \texttt{RainbowAgent.num\_atoms} & \texttt{51} \\
    \texttt{RainbowAgent.vmax} & \texttt{10.} \\
    \texttt{RainbowAgent.update\_horizon} & \texttt{3} \\
    \texttt{RainbowAgent.min\_replay\_history} & \texttt{20000}  \\
    \texttt{RainbowAgent.update\_period} & \texttt{4} \\
    \texttt{RainbowAgent.target\_update\_period} & \texttt{8000} \\
    \texttt{RainbowAgent.epsilon\_train} & \texttt{0.01} \\
    \texttt{RainbowAgent.epsilon\_eval} & \texttt{0.001} \\
    \texttt{RainbowAgent.epsilon\_decay\_period} & \texttt{250000} \\
    \texttt{RainbowAgent.replay\_schem}e & \texttt{'prioritized'} \\
    \texttt{RainbowAgent.tf\_device} & \texttt{'/gpu:0'}  \\
    \texttt{RainbowAgent.optimizer} & \texttt{@tf.train.AdamOptimizer()} \\
    & \\
    \texttt{tf.train.AdamOptimizer.learning\_rate} & \texttt{0.0000625} \\
    \texttt{tf.train.AdamOptimizer.epsilon} & \texttt{0.00015} \\
    & \\
    \texttt{HyperRainbowAgent.number\_of\_gamma} & \texttt{10} \\
    \texttt{HyperRainbowAgent.gamma\_max} & \texttt{0.99} \\
    \texttt{HyperRainbowAgent.hyp\_exponent} & \texttt{0.01} \\
    \texttt{HyperRainbowAgent.acting\_policy} & \texttt{'largest\_gamma'} \\
    \end{tabular}
    \caption{Configurations for the Multi-C51 and Multi-Rainbow used with Dopamine \cite{DBLP:journals/corr/abs-1812-06110}.}
    \label{tbl: gin_configs}
\end{table}

\section{Auxiliary Task Results}
Final results of the multi-horizon auxiliary task on Rainbow (Multi-Rainbow) in Table \ref{tbl: final_results}.

\begin{table}[h!]
    \centering
    \small
    \begin{tabular}{c c c c c c}
Game Name & DQN & C51 & Rainbow & Multi-Rainbow \\
\hline \hline
\texttt{AirRaid}&	8190.3&	9191.2&	\textbf{16941.2}&	12659.5\\
\texttt{Alien}&	2666.0&	2611.4&	3858.9&	\textbf{3917.2}\\
\texttt{Amidar}&	1306.0&	1488.2&	\textbf{2805.7}&	2477.0\\
\texttt{Assault}&	1661.6&	2079.0&	\textbf{3815.9}&	3415.1\\
\texttt{Asterix}&	3772.5&	15289.5&	19789.2&	\textbf{24385.6}\\
\texttt{Asteroids}&	844.7&	1241.5&	1524.1&	\textbf{1654.5}\\
\texttt{Atlantis}&	\textbf{935784.0}&	894862.0&	890592.0&	923276.7\\
\texttt{BankHeist}&	723.5&	863.4&	\textbf{1209.0}&	1132.0\\
\texttt{BattleZone}&	20508.5&	28323.2&	\textbf{42911.1}&	38827.1\\
\texttt{BeamRider}&	6326.4&	6070.6&	7026.7&	\textbf{7610.9}\\
\texttt{Berzerk}&	590.3&	538.3&	864.0&	\textbf{879.1}\\
\texttt{Bowling}&	40.3&	49.8&	\textbf{68.8}&	62.9\\
\texttt{Boxing}&	83.3&	83.5&	98.8&	\textbf{99.3}\\
\texttt{Breakout}&	146.6&	\textbf{254.1}&	123.9&	162.5\\
\texttt{Carnival}&	4967.9&	4917.1&	\textbf{5211.8}&	5072.2\\
\texttt{Centipede}&	3419.9&	\textbf{8068.9}&	6878.0&	6946.6\\
\texttt{ChopperCommand}&	3084.5&	6230.4&	13415.1&	\textbf{13942.9}\\
\texttt{CrazyClimber}&	113992.2&	146072.3&	151454.9&	\textbf{160161.0}\\
\texttt{DemonAttack}&	7229.2&	8485.1&	\textbf{19738.0}&	14780.9\\
\texttt{DoubleDunk}&	-4.5&	2.7&	\textbf{22.6}&	21.9\\
\texttt{ElevatorAction}&	2434.3&	73416.0&	81958.0&	\textbf{85633.3}\\
\texttt{Enduro}&	895.0&	1652.9&	2290.1&	\textbf{2337.5}\\
\texttt{FishingDerby}&	12.4&	16.6&	44.5&	\textbf{45.1}\\
\texttt{Freeway}&	26.3&	\textbf{33.8}&	33.8&	33.8\\
\texttt{Frostbite}&	1609.6&	4522.8&	\textbf{8988.5}&	7929.7\\
\texttt{Gopher}&	6685.8&	8301.1&	11749.6&	\textbf{13664.6}\\
\texttt{Gravitar}&	339.1&	709.8&	1293.0&	\textbf{1638.7}\\
\texttt{Hero}&	17548.5&	34117.8&	47545.4&	\textbf{50141.8}\\
\texttt{IceHockey}&	-5.0&	-3.3&	2.6&	\textbf{6.3}\\
\texttt{Jamesbond}&	618.3&	816.5&	\textbf{1263.8}&	773.4\\
\texttt{JourneyEscape}&	-2604.2&	-1759.1&	\textbf{-818.1}&	-1002.9\\
\texttt{Kangaroo}&	13118.1&	9419.7&	13794.0&	\textbf{13930.6}\\
\texttt{Krull}&	6558.0&	\textbf{7232.3}&	6292.5&	6645.7\\
\texttt{KungFuMaster}&	26161.2&	27089.5&	30169.6&	\textbf{31635.2}\\
\texttt{MontezumaRevenge}&	2.6&	\textbf{1087.5}&	501.3&	800.3\\
\texttt{MsPacman}&	3664.0&	3986.2&	4254.2&	\textbf{4707.3}\\
\texttt{NameThisGame}&	7808.1&	\textbf{12934.0}&	9658.9&	11045.9\\
\texttt{Phoenix}&	5893.4&	6577.3&	8979.0&	\textbf{23720.3}\\
\texttt{Pitfall}&	-11.8&	-5.3&	\textbf{0.0}&	0.0\\
\texttt{Pong}&	17.4&	19.7&	20.3&	\textbf{20.6}\\
\texttt{Pooyan}&	3800.8&	3771.2&	\textbf{6347.7}&	4670.0\\
\texttt{PrivateEye}&	2051.8&	19868.5&	\textbf{21591.4}&	888.9\\
\texttt{Qbert}&	11011.4&	11616.6&	19733.2&	\textbf{20817.4}\\
\texttt{Riverraid}&	12502.4&	13780.4&	\textbf{21624.2}&	21421.2\\
\texttt{RoadRunner}&	40903.3&	49039.8&	\textbf{56527.4}&	55613.0\\
\texttt{Robotank}&	62.5&	64.7&	\textbf{67.9}&	67.2\\
\texttt{Seaquest}&	2512.4&	38242.7&	11791.5&	\textbf{64985.0}\\
\texttt{Skiing}&	\textbf{-15314.9}&	-17996.7&	-17792.9&	-15603.3\\
\texttt{Solaris}&	2062.7&	2788.0&	3061.9&	\textbf{3139.9}\\
\texttt{SpaceInvaders}&	1976.0&	4781.9&	4927.9&	\textbf{8802.1}\\
\texttt{StarGunner}&	47174.3&	35812.4&	58630.5&	\textbf{72943.2}\\
\texttt{Tennis}&	-0.0&	\textbf{22.2}&	0.0&	0.0\\
\texttt{TimePilot}&	3862.5&	8562.7&	12486.1&	\textbf{14421.7}\\
\texttt{Tutankham}&	141.1&	253.1&	255.6&	\textbf{264.9}\\
\texttt{UpNDown}&	10977.6&	9844.8&	42572.5&	\textbf{50862.3}\\
\texttt{Venture}&	88.0&	1430.7&	1612.4&	\textbf{1639.9}\\
\texttt{VideoPinball}&	222710.4&	594468.5&	\textbf{651413.1}&	650701.1\\
\texttt{WizardOfWor}&	3150.8&	3633.8&	8992.3&	\textbf{9318.9}\\
\texttt{YarsRevenge}&	25372.0&	12534.2&	47183.8&	\textbf{49929.4}\\
\texttt{Zaxxon}&	5199.9&	7509.8&	15906.2&	\textbf{21921.3}\\
    \end{tabular}
    \caption{Multi-Rainbow agent returns versus the DQN, C51 and Rainbow agents of Dopamine \cite{DBLP:journals/corr/abs-1812-06110}.}
    \label{tbl: final_results}
\end{table}

\end{document}